\documentclass[letterpaper,10pt]{article}
\usepackage[margin=1in]{geometry} 

\usepackage{scrextend}
\usepackage[square,numbers,sort&compress]{natbib}

\usepackage{enumitem}
\usepackage[cmex10]{amsmath}
\usepackage{amssymb}
\usepackage{authblk}
\usepackage[title]{appendix}
\usepackage{amsmath}%
\usepackage{amsfonts}%
\usepackage{amssymb}%
\usepackage{amsthm}
\usepackage{breqn}
\usepackage{graphicx}

\usepackage{caption}
\usepackage{float}

\usepackage{algorithm}
\usepackage{algpseudocode}

\usepackage{colonequals}
\usepackage{psfrag}
\usepackage{subfigure}
\usepackage{mathrsfs}
\usepackage[usenames,dvipsnames]{xcolor}
\usepackage{hyperref}

\hypersetup{%
  colorlinks=true,% hyperlinks will be black
  urlcolor=blue,%
  linkcolor=blue,%
  citecolor=blue,%
  linkbordercolor=blue,% hyperlink borders will be red
  pdfborderstyle={/S/U/W 1}% border style will be underline of width 1pt
}

\usepackage{booktabs}       % professional-quality tables

\usepackage{mdframed}
\usepackage{bm}
\usepackage{bbm}
\usepackage{cite}

\usepackage{multirow}

\newtheorem{thm}{Theorem}
\newtheorem{lem}{Lemma}

\newtheorem{prop}{Proposition}

\newtheorem{cor}{Corollary}

\theoremstyle{definition}

\newtheorem{defn}{Definition}

\newtheorem{rem}{Remark}
\newtheorem{example}{Example}

\DeclareMathOperator*{\argmin}{\arg\!\min}

\newcommand{\etal}{\textit{et al.}~}

\newcommand{\diag}{\mathsf{diag}}
\newcommand{\mean}{\mathbb{E}}
\newcommand{\ExpVal}[2]{\mean{}\left[ #2 \right]}
\newcommand{\EE}[1]{\ExpVal{}{#1}}

\newcommand{\KL}{\textnormal{D}_{\mathsf{KL}}}

\renewcommand{\tilde}{\widetilde}

\newcommand{\Reals}{\mathbb{R}}

\newcommand{\defined}{\triangleq}

\newcommand{\bnu}{\pmb{\nu}}

\newcommand{\sto}{\mbox{\normalfont s.t.}}

\newcommand{\bSigma}{\mathbf{\Sigma}}

\newcommand{\indicator}[1]{\mathbb{I}{#1}}

\newcommand{\dif}{\textrm{d}}

\newcommand{\bmu}{\pmb{\mu}}

\newcommand{\RNum}[1]{\uppercase\expandafter{\romannumeral #1\relax}}

\newcommand{\TV}{\textnormal{D}_{\mathsf{TV}}}
\newcommand{\chidiv}{\textnormal{D}_{\chi^2}}
\newcommand{\FER}[2]{\mathsf{FER}_{#1}(#2)}

\providecommand{\keywords}[1]
{
  {\small	
  \textbf{\textit{Index Terms---}} #1
  }
}

\makeatletter
\newcommand{\printfnsymbol}[1]{%
  \textsuperscript{\@fnsymbol{#1}}%
}

\title{To Split or Not to Split:\\
The Impact of Disparate Treatment in Classification}

\author{Hao Wang, Hsiang~Hsu, Mario~Diaz, Flavio~P.~Calmon\thanks{
H.~Wang, H.~Hsu, and F.~P.~Calmon are with Harvard University (e-mail: \{hao\_wang,hsianghsu\}@g.harvard.edu; flavio@seas.harvard.edu).\\
M.~Diaz is with the Instituto de Investigaciones en Matem\'{a}ticas Aplicadas y en Sistemas (IIMAS), Universidad Nacional Aut\'{o}noma de M\'{e}xico (e-mail: mario.diaz@sigma.iimas.unam.mx).
}
}

\date{}

\begin{document}
\maketitle
\vspace{-.3in}

\begin{abstract}\noindent
Disparate treatment occurs when a machine learning model yields different decisions for individuals based on a sensitive attribute (e.g., age, sex). In domains where prediction accuracy is paramount, it could potentially be acceptable to fit a model which exhibits disparate treatment. To evaluate the effect of disparate treatment, we compare the performance of split classifiers (i.e., classifiers trained and deployed separately on each group) with group-blind classifiers (i.e., classifiers which do not use a sensitive attribute). We introduce the benefit-of-splitting for quantifying the performance improvement by splitting classifiers. Computing the benefit-of-splitting directly from its definition could be intractable since it involves solving optimization problems over an infinite-dimensional functional space. Under different performance measures, we (i) prove an equivalent expression for the benefit-of-splitting which can be efficiently computed by solving small-scale convex programs; (ii) provide sharp upper and lower bounds for the benefit-of-splitting which reveal precise conditions where a group-blind classifier will always suffer from a non-trivial performance gap from the split classifiers. In the finite sample regime, splitting is not necessarily beneficial and we provide data-dependent bounds to understand this effect. Finally, we validate our theoretical results through numerical experiments on both synthetic and real-world datasets.
\end{abstract}


\section{Introduction}

A machine learning (ML) model exhibits \emph{disparate treatment} \citep{barocas2016big} if it treats similar data points from distinct individuals differently based on a sensitive attribute (e.g., age, sex).
In applications such as hiring, the existence of disparate treatment can be illegal \citep{EEOC1979}. However, in settings such as healthcare, it can be legal and ethical to fit a model which presents disparate treatment in order to improve prediction accuracy \citep{cho2006racial,cohn2006use,perez2017now}. For example, the Equal Credit Opportunity Act (ECOA) permits a creditor to use an applicant's age and income for analyzing credit, as long as such information is used in a fair manner (see 12 CFR \textsection 1002.6(b)(2) in \citep{ECOA20}).

The role of a sensitive attribute in fair classification can be understood through several metrics and principles. When a ML model is deployed in practice, fairness can be quantified in terms of the performance disparity \emph{conditioned} on a sensitive attribute, such as statistical parity \citep{feldman2015certifying} and equalized odds \citep{hardt2016equality}. In domains where the goal is to predict accurately (e.g., medical diagnostics), \emph{non-maleficence} (i.e., ``do no harm'') and \emph{beneficence} (i.e., ``do good'') \citep{beauchamp2001principles} become more appropriate moral principles for fairness \citep{ustun2019fairness,martinez2019fairness,zafar2017parity}. Accordingly, a ML model should avoid the causation of harm and be as accurate as possible on each protected group.

The relationship between achieving the above-mentioned  principles and allowing a classifier to exhibit disparate treatment is complex. On the one hand, using a \emph{group-blind classifier} (i.e., a classifier that does not use the sensitive attribute as an input feature) may cause harm unintentionally since model performance relies on the distribution of the input data \citep{dwork2018decoupled,ustun2019fairness,wang2019repairing,wang2018direction}. This probability distribution can vary significantly conditioned on a sensitive attribute due to, for example, inherent differences between groups \citep{dwork2018decoupled}, differences in labeling \citep{blum2019recovering}, and differences in sampling \citep{suresh2019framework}. 
On the other hand, training a separate classifier for each protected group---a setting we refer to as \emph{splitting classifiers}---does not necessarily guarantee non-maleficence when sample size is limited \citep{zhou2017can}: groups with insufficient samples may incur a high generalization error and suffer from overfitting.

We consider two questions that are central to understanding non-maleficence and beneficence through the use of a sensitive attribute by a ML model:
\begin{enumerate}[label=(\roman*)]
\item When is it beneficial to split classifiers in terms of model performance? 
\item When splitting is beneficial, how much do the split classifiers outperform a group-blind classifier? 
\end{enumerate}
First, we show that in the information-theoretic regime where the underlying distribution is known---or, equivalently, an arbitrarily large number of samples are available---splitting \emph{never harms} any group in terms of average performance metrics. Thus, splitting will naturally follow the non-maleficence principle in the large-sample regime. 
Second, we introduce a notion called the \emph{benefit-of-splitting} which measures the performance improvement by splitting classifiers compared to using a group-blind classifier across all groups. The benefit-of-splitting is also an information-theoretic quantity as it only relies on the underlying data distribution rather than number of samples or hypothesis class.

\begin{figure}[t]
\centering
\includegraphics[width=0.6\textwidth]{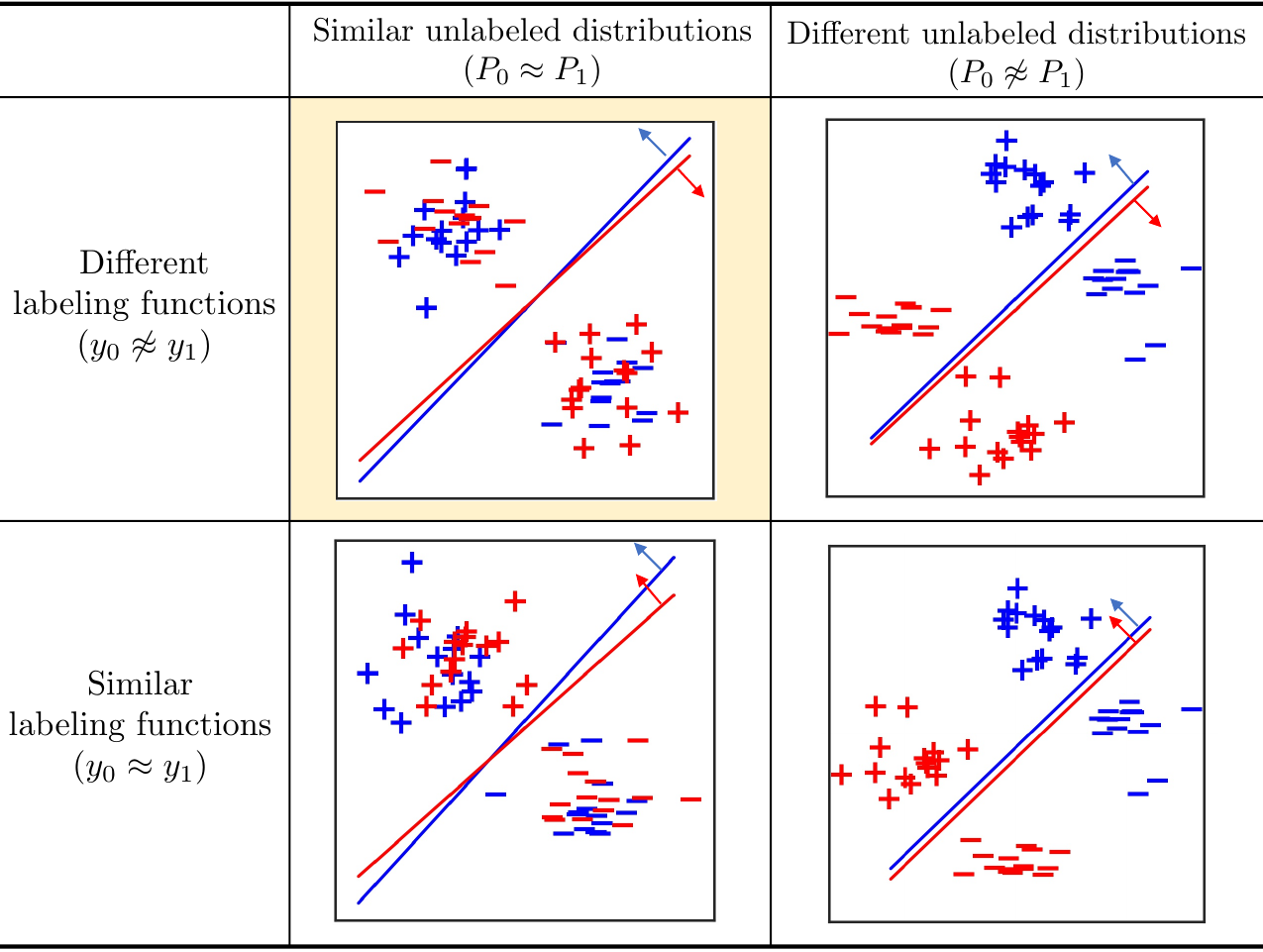}
\caption{\small{The taxonomy of splitting based on two different factors. Samples from two groups are depicted in red and blue, respectively, and their labels are represented by $+$, $-$. Each group's labeling function is shown with the corresponding color and the arrows indicate the regions where the points are labeled as $+$. Information-theoretically, splitting classifiers benefits model performance the most if the labeling functions are different and the unlabeled distributions are similar (yellow region).}
}
\label{Fig::case_illus}
\end{figure}

The definition of the benefit-of-splitting involves a model performance measure and, hence, we divide our analyses into two parts based on different choices of this measure. In Section~\ref{sec::acc_imp_split}, we quantify model performance in terms of standard loss functions (e.g., $\ell_1$ and cross entropy loss). For the benefit-of-splitting under these loss functions, we provide sharp upper and lower bounds (Theorem~\ref{thm::ben_split_l12_loss_bounds}) that capture when splitting classifiers benefits model performance the most. These bounds indicate two factors (see Figure~\ref{Fig::case_illus} for an illustration) which are central to the benefit-of-splitting: (i) disagreement between labeling functions\footnote{\label{lab_func_unlab} We borrow the terms ``labeling function'' and ``unlabeled distribution'' from the domain adaptation literature \citep{ben2010theory,mansour2009domain}. The labeling function takes a data point as an input and produces a probability of its binary label being $1$ and the unlabeled distribution is a (marginal) probability distribution of the unlabeled data. Furthermore, the labeling function can be viewed as a ``channel'' (i.e., conditional distribution) in the information theory parlance. The formal definitions are given in Section~\ref{subsec::notation_pre}.}, (ii) similarity between unlabeled distributions\footref{lab_func_unlab}. Based on these two factors, our upper bounds in Theorem~\ref{thm::ben_split_l12_loss_bounds} indicate that splitting does not produce much benefit if the labeling functions are similar or the unlabeled distributions are different; our lower bounds in Theorem~\ref{thm::ben_split_l12_loss_bounds} indicate that splitting benefits the most if two groups' labeling functions are different and unlabeled distributions are similar. Furthermore, our lower bounds in Theorem~\ref{thm::ben_split_l12_loss_bounds} lead to an impossibility (i.e., converse) result for group-blind classifiers: under certain precise conditions, using a group-blind classifier will always suffer from an inherent accuracy trade-off between different groups and splitting classifiers can reconcile this issue. This converse result is information-theoretic: a data scientist cannot overcome this limit by using more samples or altering the hypothesis class.

In Section~\ref{sec::perf_imp_split}, we consider false error rate as a performance measure since in applications such as medical diagnostics, high false error rate could result in unintentional harm \citep{lalkhen2008clinical}. Under this metric, computing the benefit-of-splitting directly from its definition may at first seem intractable since it involves an optimization  over an infinite-dimensional functional space. Nonetheless, we prove that the benefit-of-splitting under false error rate has an equivalent, dual expression (Theorem~\ref{thm::FER_split_dual}) which only requires solving two small-scale convex programs. Furthermore, the objective functions of these convex programs have closed-form supergradients (Proposition~\ref{prop::part_deri_FER}). Combining these two results leads to an efficient procedure (Algorithm~\ref{alg:descent}) for computing the benefit-of-splitting. 
We validate our procedure through numerical experiments on synthetic datasets in Section~\ref{subsec::syn_data}. When the underlying data distribution is known, our procedure has a provable convergence guarantee and returns the precise values of the benefit-of-splitting. When the underlying data distribution is unknown, our procedure may suffer from  approximation errors but still outperforms more naive empirical approaches.

The aforementioned results capture the benefit-of-splitting from an information-theoretic perspective where the underlying data distributions are assumed to be known and the space of potential classifiers is unrestricted. In Section~\ref{sec::split_practice}, we consider the effect of splitting classifiers in a more practical setting where group-blind and split classifiers are restricted over the same hypothesis class (e.g., logistic regressions) and the underlying distribution is accessed only through finitely many i.i.d. samples. In this case, splitting classifiers is not necessarily beneficial since the group with less samples may suffer from overfitting. 
To quantify the effect of splitting classifiers, we analyze the \emph{sample-limited benefit-of-splitting}. We derive upper and lower bounds for the benefit-of-splitting in this regime in Theorem~\ref{thm::samp_comp}. These bounds  disentangle three factors which determine the effect of splitting classifiers in practice: (i) disagreement between optimal (split) classifiers and training error associated with these optimal classifiers; (ii) similarity between (empirical) unlabeled distributions; and (iii) model complexity and number of samples. 
The first two factors are analogous to the ones that affect the benefit-of-splitting in the information-theoretic regime: when the hypothesis class is complex enough and the sample size tends to infinity, the optimal classifiers approximate the labeling functions and the empirical unlabeled distributions converge to the true unlabeled distributions. 
Finally, we illustrate how these factors determine the performance impact of splitting classifiers through experiments on 40 datasets downloaded from OpenML~\citep{OpenML2013}.

% proof tech.
The proof techniques of this paper are based on fundamental tools found in statistics, such as Brown-Low's two-points lower bound \citep{brown1996constrained}, and methods in convex analysis, such as Ky Fan's min-max theorem \citep{fan1953minimax}. These tools are widely used in applications such as non-parametric estimation \citep{tsybakov2008introduction}, and are useful for analyzing the min-max risk in statistical settings \citep{wu2016minimax,jiao2018minimax,polyanskiy2019dualizing,xu2016information,duchi2013local}. Furthermore, the factors that we provide for understanding the effect of splitting classifiers are inspired by the necessary and sufficient conditions of domain adaptation learnability in Ben-David~\etal~\citep{ben2010impossibility}.

%organization
The rest of this paper is organized as follows. In the remainder of this section, we review related works and present notation adopted in this paper. In Section~\ref{sec::cost_split}, we introduce the main object of interest: the benefit-of-splitting. Under different performance measures, we provide upper and lower bounds for the benefit-of-splitting in Section~\ref{sec::acc_imp_split} and present an efficient procedure for computing the benefit-of-splitting in Section~\ref{sec::perf_imp_split}. The effect of splitting classifiers in the finite sample regime is studied in Section~\ref{sec::split_practice}. Finally, we illustrate our results through numerical experiments in Section~\ref{sec::exp} and provide conclusion remarks and future works in Section~\ref{sec:conc}.

\subsection{Related Work}
\label{subsec::RelatedWork}

\paragraph{Privacy.} 

Fairness and privacy are closely connected and central to trustworthy machine learning. In this paper, we study the impact of disparate treatment from an information-theoretic perspective: we assume the underlying data distribution is known and analyze how the different distributions between groups affect the performance improvement by splitting classifiers. 
In this regard, our present work relates with studies on information-theoretic privacy \citep[see e.g.,][]{li2018information,rassouli2019optimal,rebollo2009t,basciftci2016privacy,takbiri2018privacy,nageswaran2019data,liao2017hypothesis}. These efforts explore the fundamental limits of privacy-utility trade-offs by also assuming the underlying data distribution is known. 
For example, Makhdoumi~\etal~\citep{makhdoumi2014information} introduce privacy funnel method for solving the privacy-utility trade-offs and connect it with the information bottleneck \citep{tishby2000information}, and this connection is further studied in \citep{hsu2018generalizing}. Kairouz~\etal~\citep{kairouz2014extremal} study the trade-offs between local differential privacy \citep{warner1965randomized,duchi2013local,geumlek2019profile} and utility functions measured by $f$-divergence \citep{csiszar1967information}. 
Besides analyzing the fundamental limits, there are works \citep[see e.g.,][]{du2012privacy,bertran2019adversarially,wang2019privacy,hsu2019information} on designing privacy mechanisms which enable a certain level of utility to be obtained from the disclosed datasets while controlling private information leakage. The robustness of the privacy mechanisms is analyzed in \citep{issa2019operational,diaz2019robustness} when these privacy mechanisms are constructed by using finitely many samples. We follow a similar line of analysis in order to understand the effect of splitting classifiers in the finite sample regime and complement our bounds for the benefit-of-splitting by incorporating additional factors such as sample size and model complexity.

\paragraph{Domain adaptation.} A standard assumption in ML is that the training and testing data are drawn from the same underlying probability distribution. Domain adaptation \citep{mansour2009domain,ben2010theory,ganin2016domain} and transfer learning \citep{kuzborskij2013stability,kpotufe2018marginal} consider a more general setting where models are trained on a source domain and deployed on a (different) target domain. A common assumption therein is known as \emph{covariate shift}, which requires the source and target domain share the same labeling function. In this paper, we prove  (see Theorem~\ref{thm::ben_split_l12_loss_bounds}) that if the covariate shift assumption is violated and two groups' unlabeled distributions are similar, then no classifier can perform well on both groups. In this regard, our work is connected to Ben-David~\etal \citep{ben2010impossibility} which present impossibility results on domain adaptation learnability. Compared to \citep{ben2010impossibility}, Theorem~\ref{thm::ben_split_l12_loss_bounds} characterizes an information-theoretic fundamental limit which cannot be circumvented by using a large number of samples or a carefully designed hypothesis class. Furthermore, the lower bound in Theorem~\ref{thm::ben_split_l12_loss_bounds} serves as a complementary statement to the upper bounds in domain adaptation \citep[cf.][]{mansour2009domain,ben2010theory}. These bounds jointly describe the range of the loss a data scientist may incur by training a model on the source domain and deploying on the target domain.

\paragraph{Fair ML.}

ML models have been increasingly used in applications of individual-level consequences, ranging from recidivism prediction \citep{angwin2016machine} and lending \citep{hunt2005redlining} to healthcare \citep{kinyanjui2019estimating}. A number of works in fair ML aim at understanding why discrimination happens \citep{kleinberg2016inherent,datta2016algorithmic,chouldechova2017fair,chen2018my,jiang2018trust,kallus2018residual,adler2018auditing,dutta2020information,cotter2019making}; how it can be quantified \citep{pleiss2017fairness,berk2018fairness,chen2019fairness}; and how it can be reduced \citep{kamiran2012data,calmon2017optimized,menon2018cost,agarwal2018reductions,kearns2018preventing,ghassami2018fairness,hashimoto2018fairness,kim2019multiaccuracy,celis2019classification,alghamdi2020model}. There are also an increasing number of studies that take causality into account for understanding and mitigating discrimination \citep{kusner2017counterfactual,kilbertus2017avoiding,nabi2018fair,hu2018welfare,chiappa2019path}. 
We build on a line of recent results on decoupling predictive models for improving accuracy-fairness trade-offs \citep[see e.g.,][]{zafar2017parity,dwork2018decoupled,ustun2019fairness,lipton2018does,kleinberg2018algorithmic}. For example, Ustun~\etal~\citep{ustun2019fairness} introduce a tree structure to recursively choose sensitive attributes for decoupling. Lipton~\etal~\citep{lipton2018does} show that using group-blind classifiers could be suboptimal in terms of trading off accuracy and fairness. The work closest to ours is Dwork~\etal~\citep{dwork2018decoupled} which present a decoupling technique to learn separate models for different groups. A detailed comparison with \citep{dwork2018decoupled} is given in Section~\ref{subsec:comparison}. 

\subsection{Notation and Definitions}
\label{subsec::notation_pre}

Consider a binary classification task (e.g., detecting pneumonia from X-rays) where the goal is to learn a probabilistic classifier $h:\mathcal{X}\to [0,1]$ that predicts a label (e.g., presence of pneumonia) $Y\in \{0,1\}$ using input features (e.g., chest X-rays) $X\in\mathcal{X}$.
We assume there is an additional binary\footnote{For the sake of illustration, we assume that the sensitive attribute $S$ is binary but our results can be extended to a setting of multi-groups. Furthermore, split classifiers can be applied to a scenario where multiple subgroups overlap \citep{kearns2018preventing,blum2019advancing} since individuals belonging to both groups can opt for either one of the split classifiers.} sensitive attribute  (e.g., sex) $S\in \{0,1\}$ that does not belong to the input features $X$. 
We denote the unlabeled probability distributions of input features conditioned on the sensitive attribute by 
\begin{align*}
    P_0 \defined P_{X|S=0},\quad P_1\defined P_{X|S=1}.
\end{align*}
The labeling functions of the two groups are denoted by
\begin{align*}
    y_0(x)\defined P_{Y|X,S}(1|x,0),\quad y_1(x)\defined P_{Y|X,S}(1|x,1).
\end{align*}
In order to measure the difference between two unlabeled distributions (i.e., $P_0$ and $P_1$), we recall Csisz{\'a}r's $f$-divergence~\citep{csiszar1967information}. Let $f:(0,\infty)\to \Reals$ be a convex function with $f(1)=0$ and $P$, $Q$ be two probability distributions over $\mathcal{X}$. The $f$-divergence between $P$ and $Q$ is defined by 
\begin{align}
    \textnormal{D}_{f}(P\|Q) \defined \int_{\mathcal{X}} f\left(\tfrac{\dif P}{\dif Q}\right) \dif Q.
\end{align}
Some examples of $f$-divergence are included in Appendix~\ref{append::f-divergence}.

The proofs of some of our main results (Lemma~\ref{lem::ben_split_l12_loss_equivalent} and Theorem~\ref{thm::FER_split_dual}) rely on Ky Fan's min-max theorem~\citep{fan1953minimax}. As a reminder, a function $f:\mathcal{X}\times \mathcal{Y}\to \Reals$ is said to be concave-like on $\mathcal{X}$ if, for any two elements $x_1,x_2\in \mathcal{X}$ and $\lambda \in [0,1]$, there exists an element $x_0\in \mathcal{X}$ such that for all $y\in \mathcal{Y}$
\begin{align*}
    f(x_0,y)\geq \lambda f(x_1,y) + (1-\lambda)f(x_2,y). 
\end{align*}
Similarly, $f$ is said to be convex-like on $\mathcal{Y}$, if for any two elements $y_1,y_2\in \mathcal{Y}$ and $\lambda \in [0,1]$, there exists an element $y_0 \in \mathcal{Y}$ such that for all $x\in \mathcal{X}$
\begin{align*}
    f(x,y_0)\leq \lambda f(x,y_1) + (1-\lambda)f(x,y_2). 
\end{align*}
A function $g:\mathcal{X}\to \Reals$ is called upper semicontinuous on a metric space $\mathcal{X}$ if for every point $x_0\in \mathcal{X}$, $\limsup_{x\to x_0} g(x)\leq g(x_0)$. 
Next, we recall\footnote{We apply Ky Fan's min-max theorem to the function $-f$ instead of $f$.} Ky Fan's min-max theorem~\citep{fan1953minimax}. 
\begin{lem}[{\citep[Theorem~2]{fan1953minimax}}]
\label{lem::KyFan}
Let $\mathcal{X}$ be a compact Hausdorff space and $\mathcal{Y}$ an arbitrary set (not topologized). Let $f$ be a real-valued function on $\mathcal{X}\times \mathcal{Y}$ such that, for every $y\in \mathcal{Y}$, $f(\cdot,y)$ is upper semicontinuous on $\mathcal{X}$. If $f$ is concave-like on $\mathcal{X}$ and convex-like on $\mathcal{Y}$, then 
\begin{align}
    \inf_{y\in\mathcal{Y}}\max_{x\in\mathcal{X}} f(x,y)
    = \max_{x\in\mathcal{X}}\inf_{y\in\mathcal{Y}} f(x,y). 
\end{align}
\end{lem}
\section{The Benefit-of-Splitting}
\label{sec::cost_split}

We study the impact of disparate treatment by comparing the performance between optimal group-blind and split classifiers. Recall that a ML model exhibits disparate treatment if it explicitly uses a sensitive attribute to produce an output. We illustrate the difference between group-blind and split classifiers through the example of logistic regressions:
\begin{itemize}
    \item a group-blind classifier does not use a sensitive attribute as an input: $h(x) = \mathsf{logistic}(w^T x)$ where $\mathsf{logistic}(t) \defined 1/(1+\exp(-t))$ for $t\in \Reals$;
    \item split classifiers are a set of classifiers trained and deployed separately on each group: $h_s(x) = \mathsf{logistic}(w_s^Tx)$ for $s\in \{0,1\}$.
\end{itemize}    
We measure the performance of both group-blind and split classifiers in terms of the \emph{disadvantaged group} (i.e., the group with worst performance). For a given performance measure $L_s(\cdot)$ (higher values indicate a worse performance), the performance of a group-blind classifier $h$ and a set of split classifiers $\{h_s\}_{s\in \{0,1\}}$, respectively, is measured by 
\begin{align*}
    \max_{s\in \{0,1\}} L_s(h)\quad \text{and}\quad \max_{s\in \{0,1\}} L_s(h_s).
\end{align*}
Consequently, the optimal group-blind and split classifiers (across all measurable functions from $\mathcal{X}$ to $[0,1]$) achieve the performance
\begin{align*}
    \inf_{h:\mathcal{X}\to [0,1]}\max_{s\in \{0,1\}} L_s(h) \quad \text{and} \quad \max_{s\in \{0,1\}}\inf_{h:\mathcal{X}\to [0,1]} L_s(h).
\end{align*}
Next, we introduce the benefit-of-splitting to quantify the effect of splitting classifiers compared to using a group-blind classifier.
\begin{defn}
\label{defn::main_coupling_cost}
For each $s \in \{0,1\}$, let $P_{X,Y|S=s}$ be a fixed probability distribution and $L_s(\cdot)$ be a performance measure, we define the benefit-of-splitting as 
\begin{align}
\label{eq::eps_split}
    &\epsilon_{\text{split}} 
    \defined \inf_{h: \mathcal{X}\to [0,1]} \max_{s\in \{0,1\}} L_s(h) - \max_{s\in \{0,1\}}\inf_{h:\mathcal{X}\to [0,1]} L_s(h),
\end{align}
where the infimum is taken over all (measurable) functions.
\end{defn}
The benefit-of-splitting is the difference between the performance of the optimal group-blind and split classifiers. In other words, if $h^*$ and $\{h_s^*\}_{s\in \{0,1\}}$ are optimal group-blind and split classifiers respectively, i.e.,
\begin{align*}
    h^* \in \argmin_{h:\mathcal{X}\to [0,1]}\max_{s\in\{0,1\}}L_s(h),\quad
    h_s^* \in \argmin_{h:\mathcal{X}\to [0,1]} L_s(h) \ \ s\in \{0,1\},
\end{align*}
the benefit-of-splitting can be equivalently expressed as
\begin{align}
    \epsilon_{\text{split}} = \max_{s\in\{0,1\}}L_s(h^*) - \max_{s\in\{0,1\}}L_s(h_s^*).
\end{align}
In practice, a data scientist may restrict the type of classifiers by fixing a hypothesis class (e.g., logistic regressions). The benefit-of-splitting can be adapted for capturing the effect of splitting classifiers in this case (see Definition~\ref{defn::hypothesis_splitting}).

By the optimality of $h_s^*$ and the max-min inequality, we have $L_s(h^*)\geq L_s(h_s^*)$ for $s\in \{0,1\}$ and $\epsilon_{\text{split}} \geq 0$ which implies that, information-theoretically, using a separate classifier on each group will never diminish model performance compared to using a group-blind classifier. A natural question is: how much performance improvement does splitting classifiers bring? Before answering this question, we specify performance measures of interest and present the benefit-of-splitting under these performance measures.

\subsection{Loss Reduction by Splitting}

The first type of performance measures contains standard loss functions which have been widely used in fair ML \citep[see e.g.,][]{dwork2018decoupled} and domain adaptation \citep[see e.g.,][]{ben2010theory}. These loss functions quantify the disagreement between the labeling function $y_s$ and the probabilistic classifier $h$. We recast the benefit-of-splitting under these loss functions below.
\begin{defn}
\label{defn::ben_split_l1_2}
The $\ell_1$-benefit-of-splitting $\epsilon_{\text{split},1}$ is the benefit-of-splitting in Definition~\ref{defn::main_coupling_cost} with the performance measure:
\begin{align*}
    L_s(h) = \EE{\lvert h(X)-y_s(X)\rvert \mid S=s}.
\end{align*}
The $\ell_2$-benefit-of-splitting $\epsilon_{\text{split},2}$ is the benefit-of-splitting in Definition~\ref{defn::main_coupling_cost} with the performance measure:
\begin{align*}
    L_s(h) = \EE{(h(X)-y_s(X))^2 \mid S=s}.
\end{align*}
The KL-benefit-of-splitting $\epsilon_{\text{split},\mathsf{KL}}$ is the benefit-of-splitting in Definition~\ref{defn::main_coupling_cost} with the performance measure:
\begin{align*}
    L_s(h) = \EE{\KL(y_s(X)\|h(X)) \mid S=s},
\end{align*}
where for $p,q\in [0,1]$, $\KL(p\| q) \defined p\log(p/q) + (1-p)\log((1-p)/(1-q))$.
\end{defn}
\begin{rem}
\label{rem::cross_ent_split}
Another widely used loss function is cross entropy $L_s(h) = \EE{\mathsf{H}(y_s(X), h(X)) \mid S=s}$ where for $p,q\in [0,1]$, $\mathsf{H}(p,q) \defined -p\log q - (1-p)\log(1-q)$. Since $\mathsf{H}(p,q) = \KL(p\|q) + \mathsf{H}(p)$, the analysis of the benefit-of-splitting under cross entropy is essentially the same as the analysis of $\epsilon_{\text{split},\mathsf{KL}}$ (see Appendix~\ref{appendix::cross_ent}).
\end{rem}

\subsection{False Error Rate Reduction by Splitting}
\label{subsec::FER_BOS}

Now we use the false error rate (FER) as a performance measure. The false error rate of a classifier is the maximum\footnote{Our analysis can be extended to any convex combination of false positive rate and false negative rate.} between (generalized) false positive rate and (generalized) false negative rate \citep{pleiss2017fairness}. In healthcare, assuring low false error rate is as important as guaranteeing high accuracy since patients could suffer from harm due to a classifier's false error rate \citep{lalkhen2008clinical}. For example, the false negative diagnosis may delay treatment in patients who are critically ill; the false positive diagnosis could lead to an unnecessary treatment. Furthermore, a classifier with high accuracy does not necessarily mean it has low false error rate. Hence, we consider how split classifiers reduce the false error rate by recasting the benefit-of-splitting under this performance measure. 
\begin{defn}
\label{defn::FERben_split}
The FER-benefit-of-splitting $\epsilon_{\text{split},\mathsf{FER}}$ is the benefit-of-splitting in Definition~\ref{defn::main_coupling_cost} with the performance measure:
\begin{align}
\label{eq::bos_defn_fer}
    L_s(h) = \max\left\{\EE{h(X)\mid Y=0, S=s},\ \EE{1-h(X)\mid Y=1, S=s}\right\}.
\end{align}
\end{defn}

\paragraph{Connection with equalized odds.}

Equalized odds, discussed by Hardt~\etal~\citep{hardt2016equality}, is a commonly used group fairness measure that requires different  groups to have (approximately) the same false positive and false negative rates. Specifically, a probabilistic classifier $h:\mathcal{X}\to [0,1]$ satisfies equalized odds \citep{hardt2016equality,pleiss2017fairness} if 
\begin{align*}
    \EE{h(X) \mid Y=0, S=0} 
    &= \EE{h(X) \mid Y=0, S=1}\quad
    &&(\text{equal false positive rate}),\\
    \EE{1 - h(X) \mid Y=1, S=0} 
    &=\EE{1 - h(X) \mid Y=1, S=1}\quad
    &&(\text{equal false negative rate}).
\end{align*}
Under this definition, classifiers are considered ``unfair'' if their false positive rate or false negative rate vary across different groups. However, imposing equalized odds constraints may lead to a significant performance reduction in classification  \citep{chouldechova2017fair,friedler2016possibility,corbett2018measure,zhao2019inherent}. In contrast, the benefit-of-splitting definition studied in this work aims to capture the principles of non-maleficence and beneficence \citep{beauchamp2001principles}: classifiers should avoid the causation of harm and achieve the best performance on each group. 
By taking the optimal group-blind classifier as a baseline approach, this may allow split classifiers to potentially exhibit performance disparities between groups---as long as the split classifiers do not perform worse than the baseline approach and are as accurate as possible.

\section{The Taxonomy of Splitting}
\label{sec::acc_imp_split}

In this section, we analyze the loss reduction by splitting classifiers compared to using a group-blind classifier. We achieve this goal by upper and lower bounding the benefit-of-splitting under different loss functions (see Definition~\ref{defn::ben_split_l1_2}). These bounds reveal factors which could impact the effect of splitting classifiers and lead to a taxonomy of splitting, i.e., a characterization of when splitting benefits model performance the most or splitting does not bring much benefit.

Before stating the main result (i.e., bounds for the benefit-of-splitting), we prove a lemma first which converts the definition of the benefit-of-splitting into a single variable optimization problem. This lemma will be used in the proof of our lower bounds.
\begin{lem}
\label{lem::ben_split_l12_loss_equivalent}
The benefit-of-splitting under different loss functions in Definition~\ref{defn::ben_split_l1_2} have equivalent expressions
\begin{align*}
    &\epsilon_{\text{split},1} 
    = \sup_{\omega \in [0,1]} (1-\omega) \int_{\mathcal{A}_{\omega}} |y_1(x)-y_0(x)| \dif P_1(x) + \omega \int_{\mathcal{A}_{\omega}^{c}} |y_1(x)-y_0(x)|\dif P_0(x),\\
    &\epsilon_{\text{split},2} 
    = \sup_{\omega\in[0,1]} \omega(1-\omega) \int \frac{(y_1(x)-y_0(x))^2 \dif P_0(x)\dif P_1(x)}{\omega \dif P_0(x) + (1-\omega) \dif P_1(x)},\\
    &\epsilon_{\text{split},\mathsf{KL}}
    = \sup_{\omega\in[0,1]}\mathsf{JS}_{\omega}(P_{X,Y|S=0} \| P_{X,Y|S=1}) - \mathsf{JS}_{\omega}(P_{0} \| P_{1}),
\end{align*}
where $\mathcal{A}_{\omega} \defined \left\{x\in \mathcal{X} \mid \frac{\dif P_0(x)}{\dif P_1(x)} \geq \frac{1-\omega}{\omega} \right\}$ and $\mathsf{JS}_{\omega}(\cdot\|\cdot)$ is the Jensen-Shannon divergence.
\end{lem}
\begin{proof}
See Appendix~\ref{appendix::split_acc_loss_equivalent}.
\end{proof}

Next, we provide upper and lower bounds for $\epsilon_{\text{split},1}$, $\epsilon_{\text{split},2}$, and $\epsilon_{\text{split},\mathsf{KL}}$, respectively. These bounds rely on two main factors: (i) disagreement between different groups' labeling functions and (ii) similarity between their unlabeled distributions. In particular, the second factor is captured by a certain $f$-divergence~\citep{csiszar1967information} (see Appendix~\ref{append::f-divergence} for some examples of $f$-divergence).

\begin{thm}
\label{thm::ben_split_l12_loss_bounds}
The $\ell_1$-benefit-of-splitting can be upper and lower bounded 
\begin{align*}
    &\epsilon_{\text{split},1}
    \leq \min\left\{\min_{s\in\{0,1\}} \sqrt{\EE{(y_1(X)-y_0(X))^2\mid S=s}}\cdot \sqrt{1-\TV(P_0\| P_1)},\right.\\
    &\left. \quad \quad \quad \quad \quad \quad\ \frac{1}{2} \max_{s\in \{0,1\}} \EE{\lvert y_1(X)-y_0(X)\rvert \mid S=s}\right\},\\
    &\epsilon_{\text{split},1}
    \geq \frac{1}{2} \max_{s\in \{0,1\}} \left\{\EE{\lvert y_1(X)-y_0(X)\rvert \mid S=s} - \sqrt{\EE{(y_1(X)-y_0(X))^2 \mid S=s}} \cdot d_2(P_{1-s}\|P_s)\right\},
\end{align*}
where $\TV(P_0\| P_1)$ is the total variation distance and $d_2(P_{1-s}\|P_s)$ is Marton's divergence. Suppose that 
\begin{align}
    \EE{\lvert y_1(X) - y_0(X)\rvert \mid S=1}
    \geq \EE{\lvert y_1(X) - y_0(X)\rvert \mid S=0}.
\end{align}
Then the $\ell_2$-benefit-of-splitting can be upper and lower bounded 
\begin{align*}
    &\epsilon_{\text{split},2}
    \leq \min\left\{\min_{s\in \{0,1\}} \sqrt{\EE{(y_1(X)-y_0(X))^4 \mid S=s}} \cdot \sqrt{1-\TV(P_0\| P_1)},\right.\\
    &\left. \quad \quad \quad \quad \quad \quad\ \frac{1}{4}\max_{s\in \{0,1\}} \EE{(y_1(X)-y_0(X))^2 \mid S=s}\right\},\\
    &\epsilon_{\text{split},2}
    \geq \left(\frac{\EE{\lvert y_1(X) - y_0(X)\rvert \mid S=0} }{\sqrt{\chidiv(P_{1} \| P_{0}) + 1} + 1} \right)^2,
\end{align*}
where $\chidiv(P_{1} \|P_{0})$ is the chi-square divergence.  The KL-benefit-of-splitting can be upper and lower bounded
\begin{align*}
    &\epsilon_{\text{split},\mathsf{KL}} 
    \leq \min\left\{2\mathsf{JS}(P_{X,Y|S=0} \| P_{X,Y|S=1}) - 2\mathsf{JS}(P_{0} \| P_{1}),\ \max_{s\in \{0,1\}} \EE{\KL\left(y_s(X) \| \tfrac{y_0(X) + y_1(X)}{2}\right) \mid S=s} \right\},\\
    &\epsilon_{\text{split},\mathsf{KL}} \geq \mathsf{JS}(P_{X,Y|S=0} \| P_{X,Y|S=1}) - \mathsf{JS}(P_{0} \| P_{1}),
\end{align*}
where $\mathsf{JS}(\cdot\|\cdot)$ is the Jensen–Shannon divergence.
\end{thm}
\begin{proof}
See Appendix~\ref{sec::proof_bounds_eps_12}.
\end{proof}
Now we consider extreme scenarios to verify the sharpness of the bounds and to understand when splitting classifiers benefits model performance the most (see Figure~\ref{Fig::case_illus} for an illustration). 
\begin{itemize}
    \item Consider the setting where two groups share the same labeling function (i.e., $y_0 = y_1$). All the upper and lower bounds in Theorem~\ref{thm::ben_split_l12_loss_bounds} for the benefit-of-splitting under different loss functions become zero and, hence, the bounds are sharp. This is quite intuitive as one can use the labeling function $y_0$ as a group-blind classifier and it achieves perfect performance on both groups. Hence, there is no benefit of splitting classifiers. 
    
    \item Consider the setting where two groups share the same unlabeled distribution (i.e., $P_0 = P_1$). The upper and lower bounds of $\epsilon_{\text{split},1}$ are both $\EE{|y_1(X)-y_0(X)|}/2$, which is equal to $\epsilon_{\text{split},1}$. The bounds of $\epsilon_{\text{split},2}$ become
    \begin{align*}
        \frac{1}{4} \EE{|y_1(X) - y_0(X)|}^2 \leq \epsilon_{\text{split},2} \leq \frac{1}{4} \EE{(y_1(X) - y_0(X))^2}.
    \end{align*}
    If, in addition, $|y_0(x)-y_1(x)|$ is the same across all $x$, the upper and lower bounds become the same and, hence, are sharp. Finally, the bounds of $\epsilon_{\text{split},\mathsf{KL}}$ become 
    \begin{align*}
        \EE{\mathsf{JS}(y_0(X) \| y_1(X))}
        \leq \epsilon_{\text{split},\mathsf{KL}} 
        \leq \max_{s\in \{0,1\}} \EE{\KL\left(y_s(X) \| \tfrac{y_0(X) + y_1(X)}{2}\right)}.
    \end{align*}
    If, in addition, $\EE{\KL\left(y_0(X) \| (y_0(X) + y_1(X))/2\right)} = \EE{\KL\left(y_1(X) \| (y_0(X) + y_1(X))/2\right)}$, then the upper and lower bounds are equal. This extreme case indicates that when different groups have the same unlabeled distribution (i.e., $P_0=P_1$), the benefit-of-splitting is determined by the disagreement between their labeling functions (i.e., large disagreement leads to high benefit). 

    \item Consider the setting where two groups have unlabeled distributions lying on disjoint support sets. In this case, $\TV(P_0\|P_1) = 1$ and $\mathsf{JS}(P_0\|P_1) = \log2$. Hence, the upper bounds of $\epsilon_{\text{split},1}$ and $\epsilon_{\text{split},2}$ become zero. Furthermore,
    \begin{align*}
        0\leq \mathsf{JS}(P_{X,Y|S=0} \| P_{X,Y|S=1}) - \mathsf{JS}(P_{0} \| P_{1}) = \mathsf{JS}(P_{X,Y|S=0} \| P_{X,Y|S=1}) - \log2 \leq 0.
    \end{align*}
    where the last step is because the Jensen–Shannon divergence is always upper bounded by $\log 2$. Therefore, the upper bound of $\epsilon_{\text{split},\mathsf{KL}}$ is zero as well. In other words, there is no benefit of splitting classifiers when the unlabeled distributions are mutually singular. One can interpret this fact by considering a special group-blind classifier which mimics the labeling function of each group in the region where its unlabeled distribution lies. This classifier achieves perfect performance for each group. Note that such a group-blind classifier exists since we do not restrict the space of potential classifiers and, hence, any (measurable) function could become a classifier. 
\end{itemize}
To summarize, from an information-theoretic perspective, splitting classifiers benefits the most if two groups have similar unlabeled distributions and different labeling functions. This taxonomy of splitting appears for all the commonly used loss functions (i.e., $\ell_1$, $\ell_2$, and KL loss).

Recall that the benefit-of-splitting (see Definition~\ref{defn::main_coupling_cost}) measures the performance improvement by using the optimal split classifiers compared to deploying the optimal group-blind classifier across all groups. Here model performance is quantified in terms of the disadvantaged group (i.e., the group with the worst performance). We end this section by considering the Bayes risk as an alternative way of measuring model performance\footnote{For the sake of illustration, in what follows we only consider the $\ell_1$ loss.}. Specifically, the performance of a group-blind classifier $h$ and a set of split classifiers $\{h_s\}_{s\in \{0,1\}}$, respectively, is measured by 
\begin{align*}
    &\text{group-blind}:\quad &&\Pr(S=0)\cdot \EE{\lvert h(X)-y_0(X)\rvert \mid S=0} + \Pr(S=1)\cdot \EE{\lvert h(X)-y_1(X)\rvert \mid S=1},\\
    &\text{split}:\quad &&\Pr(S=0)\cdot \EE{\lvert h_0(X)-y_0(X)\rvert \mid S=0} + \Pr(S=1)\cdot \EE{\lvert h_1(X)-y_1(X)\rvert \mid S=1}.
\end{align*}
They can be equivalently written as 
\begin{align*}
    \EE{\lvert h(X)-y_S(X) \rvert}
    \quad
    \text{and}
    \quad
    \EE{\lvert h_S(X)-y_S(X) \rvert}.
\end{align*}
The performance difference between the optimal group-blind and split classifiers leads to the following definition.
\begin{defn}
We define the population-benefit-of-splitting as
\begin{align*}
    \epsilon_{\text{split},\text{pop}}
    \defined \inf_{h: \mathcal{X}\to [0,1]} \EE{|h(X)-y_S(X)|} - \inf_{\substack{h_s:\mathcal{X}\to [0,1]\\ \text{for } s\in\{0,1\}}} \EE{|h_S(X)-y_S(X)|}.
\end{align*}
\end{defn}
The population-benefit-of-splitting is upper bounded by the benefit-of-splitting (i.e., $\epsilon_{\text{split},\text{pop}} \leq \epsilon_{\text{split},1}$) since the Bayes risk is upper bounded by the worst-case risk and the split classifiers $\{y_s\}_{s\in\{0,1\}}$ composed by the labeling functions can achieve zero risk. Hence, the upper bound of $\epsilon_{\text{split},1}$ in Theorem~\ref{thm::ben_split_l12_loss_bounds} naturally translates into an upper bound of $\epsilon_{\text{split},\text{pop}}$. Next, we provide alternative bounds for $\epsilon_{\text{split},\text{pop}}$ which reveal an additional factor influencing $\epsilon_{\text{split},\text{pop}}$. 
\begin{prop}
\label{prop::bounds_popu_ben_split}
Assume $\Pr(S=0)\leq 0.5$. The population-benefit-of-splitting can be upper and lower bounded
\begin{align*}
    \epsilon_{\text{split},\text{pop}}
    &\leq \Pr(S=0)\cdot \EE{\lvert y_1(X) - y_0(X)\rvert \mid S=0},\\
    \epsilon_{\text{split},\text{pop}}
    &\geq \Pr(S=0) \left(\EE{\lvert y_1(X) - y_0(X)\rvert \mid S=0} - E_{\frac{\Pr(S=1)}{\Pr(S=0)}}(P_0\|P_1)\right),
\end{align*}
where $E_{\frac{\Pr(S=1)}{\Pr(S=0)}}(P_0\|P_1)$ is the $E_{\gamma}$-divergence with $\gamma = \Pr(S=1)/\Pr(S=0)$.
\end{prop}
\begin{proof}
See Appendix~\ref{append::pop_ben_split}.
\end{proof}
\begin{rem}
The $E_{\gamma}$-divergence plays an important role in Bayesian statistical hypothesis testing \citep{sason2016f,liu2016Egamma}. Since $\gamma \to E_{\gamma}(P\|Q)$ is non-increasing and $E_{1}(P\|Q) = \TV(P\|Q)$ \citep{liu2016Egamma}, we can further lower bound $\epsilon_{\text{split},\text{pop}}$ by using the total variation distance
\begin{align*}
    \epsilon_{\text{split},\text{pop}} \geq \Pr(S=0) \left(\EE{\lvert y_1(X) - y_0(X)\rvert \mid S=0} - \TV(P_0\|P_1)\right). 
\end{align*}
The $E_{\gamma}$-divergence relates with the DeGroot statistical information \citep{degroot1962uncertainty} through (see Equation~(421) in~\citep{sason2016f})
\begin{align*}
    \mathcal{I}_p(P\|Q) 
    = 
    \begin{cases}
    pE_{\frac{1-p}{p}}(P\|Q) \quad &p\in (0,\frac{1}{2}]\\
    (1-p)E_{\frac{p}{1-p}}(Q\|P) \quad &p\in [\frac{1}{2},1).
    \end{cases}
\end{align*}
Hence, we can write our lower bound of $\epsilon_{\text{split},\text{pop}}$ equivalently as 
\begin{align*}
    \epsilon_{\text{split},\text{pop}} \geq
    \Pr(S=0)\cdot \EE{\lvert y_1(X) - y_0(X)\rvert \mid S=0} - \mathcal{I}_{\Pr(S=0)}(P_0\|P_1).
\end{align*}
\end{rem}

As shown in Proposition~\ref{prop::bounds_popu_ben_split}, the population-benefit-of-splitting is affected not only by the above-mentioned two factors (i.e., disagreement between labeling functions and similarity between unlabeled distributions) but also by the percentage of the minority group over the whole population. This reveals a caveat of the population-benefit-of-splitting: the minority group can be underrepresented when one designs a group-blind classifier by minimizing the loss over the whole population. In contrast, the benefit-of-splitting (see Definition~\ref{defn::ben_split_l1_2}) does not rely on the probability of the sensitive attribute and, hence, represents each group equally.  
\section{An Efficient Procedure for Computing the Effect of Splitting}
\label{sec::perf_imp_split}

In the last section, we provide upper and lower bounds for the benefit-of-splitting under different kinds of loss functions. Here, we consider a different performance measure: false error rate. It turns out that the benefit-of-splitting under false error rate, denoted by $\epsilon_{\text{split}, \mathsf{FER}}$ (see Definition~\ref{defn::FERben_split}), has an equivalent expression which leads to an efficient procedure of computing $\epsilon_{\text{split}, \mathsf{FER}}$.

Even with the knowledge of the underlying data distribution, computing the benefit-of-splitting directly from its definition is challenging. This is because the space of potential classifiers is unrestricted (i.e., any measurable function could be used as group-blind or split classifiers) and solving optimization problems over this infinite-dimensional functional space could be intractable. One may attempt to circumvent this issue by restricting the classifiers over a hypothesis class. However, this naive approach has two limitations. First, it is unclear how to choose a hypothesis class in order to compute the benefit-of-splitting reliably. We will show in Example~\ref{ex::split_dep_hyp} that different hypothesis classes could result in completely different values of the benefit-of-splitting. Second, as evidenced in \citep{zafar2017fairness}, training the optimal group-blind or split classifiers may suffer from a non-convexity issue. 

We leverage the special form of the false error rate  in \eqref{eq::bos_defn_fer} and prove an equivalent expression of $\epsilon_{\text{split}, \mathsf{FER}}$ below which can be computed by solving two small-scale convex programs. The objective functions of these convex programs have closed-form supergradients. Hence, they can be solved efficiently via standard solvers, such as (stochastic) mirror descent \citep{nemirovsky1983problem,beck2003mirror}. When the data distribution is known, our procedure returns the precise values of $\epsilon_{\text{split}, \mathsf{FER}}$ without the need of training optimal group-blind and split classifiers. The equivalent expression of $\epsilon_{\text{split}, \mathsf{FER}}$ is given in the following theorem.

\begin{thm}
\label{thm::FER_split_dual}
Assume $\Pr(Y=i,S=s)> 0$ for any $i,s\in \{0,1\}$. The FER-benefit-of-splitting $\epsilon_{\text{split}, \mathsf{FER}}$ can be equivalently written as 
\begin{align*}
    \max_{\bmu \in \Delta_4} \left\{\sum_{s\in \{0,1\}}\mu_{s,1} + \EE{\left(\sum_{s,i\in \{0,1\}} \mu_{s,i}\phi_{s,i}(X)\right)_{-}} \right\} - \max_{\substack{\bnu^{(s)} \in \Delta_2\\ \text{for } s\in\{0,1\}}} \left\{\nu^{(s)}_{1} + \EE{\left(\sum_{i\in \{0,1\}} \nu^{(s)}_{i}\phi_{s,i}(X)\right)_{-}} \right\}.
\end{align*}
Here for a positive integer $d$, $\Delta_d \defined \{z\in \Reals^d \mid \sum_{i=1}^d z_i=1,\ z_i\geq0\}$, for any $a\in \Reals$, $(a)_{-} \defined \min\{a,0\}$, $\bmu \defined (\mu_{0,0},\mu_{0,1},\mu_{1,0},\mu_{1,1})$, $\bnu^{(s)} \defined (\nu^{(s)}_0,\nu^{(s)}_1)$, and 
\begin{align}
\label{eq::phi_s_i_defn}
    \phi_{s,i}(x) 
    \defined \frac{(1-i-y_s(x))\Pr(S=s\mid X=x)}{\Pr(Y=i, S=s)},\quad s,i\in \{0,1\}.
\end{align}
\end{thm}
\begin{proof}
See Appendix~\ref{append::FER_split}.
\end{proof}
\begin{rem}
We demonstrate a proof sketch of  Theorem~\ref{thm::FER_split_dual}. The FER-benefit-of-splitting $\epsilon_{\text{split}, \mathsf{FER}}$ is composed by $\inf_{h:\mathcal{X}\to [0,1]}\max_{s\in \{0,1\}} L_s(h)$ and $\max_{s\in \{0,1\}}\inf_{h:\mathcal{X}\to [0,1]} L_s(h)$. The first term can be equivalent written as 
\begin{align}
\label{eq::rem_inf_max_fer}
    &\inf_{h:\mathcal{X}\to [0,1]} \max_{\bmu\in \Delta_4} \left\{\sum_{s\in \{0,1\}} \mu_{s,0} \EE{h(X)\mid Y=0, S=s} + \mu_{s,1} \EE{1-h(X)\mid Y=1, S=s}\right\}.
\end{align}
The key step in our proof is to swap maximum and infimum in \eqref{eq::rem_inf_max_fer} by using Ky Fan's min-max theorem \citep{fan1953minimax} (see Lemma~\ref{lem::KyFan}). Then for a fixed $\bmu$, the optimal classifier owns a closed-form expression. After some algebraic manipulations, \eqref{eq::rem_inf_max_fer} becomes the first convex program in the equivalent expression of $\epsilon_{\text{split}, \mathsf{FER}}$. 
In the same vein, the another term $\max_{s\in \{0,1\}}\inf_{h:\mathcal{X}\to [0,1]} L_s(h)$ becomes the second convex program.
\end{rem}

Next, we show that the objective functions of the convex programs in Theorem~\ref{thm::FER_split_dual} have closed-form supergradients.
\begin{prop}
\label{prop::part_deri_FER}
Under the same notations and assumptions in Theorem~\ref{thm::FER_split_dual}, functions $g:\Delta_4 \to \Reals$ and $g_s:\Delta_2 \to \Reals$ with $s\in \{0,1\}$ defined as
\begin{align*}
    &g(\bmu) \defined \sum_{s\in \{0,1\}}\mu_{s,1} + \EE{\left(\sum_{s,i\in \{0,1\}} \mu_{s,i}\phi_{s,i}(X)\right)_{-}}\quad
    g_s(\bnu) \defined \nu_{1} + \EE{\left(\sum_{i\in \{0,1\}} \nu_{i}\phi_{s,i}(X)\right)_{-}}
\end{align*}
have a closed-form supergradient, respectively:
\begin{align}
    &\partial g(\bmu) \ni \left(i + \EE{\psi_{s,i}(X) \cdot \indicator{\Big[\sum_{s',i'\in \{0,1\}} \mu_{s',i'}\phi_{s',i'}(X) < 0\Big]\ \Bigg|\ S=s}}\right)_{s,i\in \{0,1\}},\\
    &\partial g_s(\bnu) \ni \left(i + \EE{\psi_{s,i}(X) \cdot \indicator{\Big[\sum_{i'\in \{0,1\}} \nu_{i'}\phi_{s,i'}(X) < 0\Big]}\ \Bigg|\  S=s}\right)_{i\in \{0,1\}},
\end{align}
where $\indicator{[\cdot]}$ is the indicator function and 
\begin{align}
    \psi_{s,i}(x) \defined \frac{1-i-y_s(x)}{\Pr(Y=i\mid S=s)},\quad s,i\in \{0,1\}.
\end{align}
\end{prop}
\begin{proof}
See Appendix~\ref{append::deri_FER}.
\end{proof}
When the underlying data distribution is known, one can compute $\epsilon_{\text{split}, \mathsf{FER}}$ by solving the convex programs in Theorem~\ref{thm::FER_split_dual} via standard tools, such as mirror descent, with convergence guarantees \citep{beck2003mirror}. This is non-trivial because, as stated before, computing $\epsilon_{\text{split}, \mathsf{FER}}$ directly from its definition could be intractable.

In practice, when the underlying data distribution is unknown, one can first approximate the conditional distribution $\Pr(S=1\mid X=x)$ and the labeling functions $y_0(x)$, $y_1(x)$ by training three well-calibrated binary classifiers. These classifiers will be called when computing the supergradient of the objective functions (see Proposition~\ref{prop::part_deri_FER}). We summarize our procedure of computing $\epsilon_{\text{split}, \mathsf{FER}}$ in Algorithm~\ref{alg:descent} where stochastic mirror descent is used for solving the convex programs in Theorem~\ref{thm::FER_split_dual}. The numerical results are deferred to Section~\ref{subsec::syn_data}. 

\begin{algorithm}[httb]
\begingroup
\small
\caption{Computing $\epsilon_{\text{split}, \mathsf{FER}}$ using stochastic mirror descent.}
\label{alg:descent}

\begin{algorithmic}[*]
\State {\bfseries Input:}

\State \quad dataset: $\mathcal{D}=\{(x_i,y_i,s_i)\}_{i=1}^n$,\quad  maximum number of iterations: $T$,\quad step size: $\{\eta_t\}_{t=1}^T$

\State \textbf{Initialize} \vspace{0.25em}

\State \quad $\mathcal{I}_0 \gets \{i = 1, \cdots, n \mid s_i = 0\}$ \Comment{indices of points in $\mathcal{D}$ with $s_i = 0$}

\State \quad $\mathcal{D}_0 \gets (x_i, y_i)$ for $i \in \mathcal{I}_0$ \Comment{points with $s_i = 0$}

\State \quad $\mathcal{D}_1 \gets (x_i, y_i)$ for $i \not\in \mathcal{I}_0$ \Comment{points with $s_i = 1$}

\State \quad approximate $\Pr(S=1\mid X=x)$ \Comment{train a classifier using $\{(x_i,s_i)\}_{i=1}^n$}

\State \quad approximate $y_0(x)$ and $y_1(x)$ \Comment{train two classifiers using $\mathcal{D}_0$ and $\mathcal{D}_1$, respectively}

\State \quad $\bmu \gets (0.25,0.25,0.25,0.25)$ and $\bnu^{(s)} \gets (0.5,0.5)$  \Comment{initialize values}

\For{t = 1,2, $\cdots$,T}
\State draw unlabeled sample $x_{0,t}$, $x_{1,t}$ 
\Comment{randomly draw sample from $\mathcal{D}_0$, $\mathcal{D}_1$}

\State pick $\bm{w} \in \partial g(\bmu)$ and $\bm{w}^{(s)} \in \partial g_s(\bnu^{(s)})$ \Comment{approximate supergradient using $x_{0,t}, x_{1,t}$}

\State $\mu_j \gets \mu_j \exp(\eta_t w_j)/\sum_{j'} \mu_{j'} \exp(\eta_t w_{j'})$ \Comment{update variable via entropic descent}

\State $\nu^{(s)}_j \gets \nu^{(s)}_j \exp(\eta_t w^{(s)}_j)/\sum_{j'} \nu^{(s)}_{j'} \exp(\eta_t w^{(s)}_{j'})$  \Comment{update variable via entropic descent}

\EndFor

\State \textbf{return:} $g(\bmu) - \max_{s\in \{0,1\}} g_s(\bnu^{(s)})$  \Comment{the FER-benefit-of-splitting: $\epsilon_{\text{split},\mathsf{FER}}$}

\end{algorithmic}

\endgroup

\end{algorithm}
Our procedure can be understood through the following two steps:
\begin{itemize}
    \item training a classifier to identify the sensitive attribute using input features and a classifier for each group to predict label using input features;
    \item solving (convex) programs with these classifiers in hand. 
\end{itemize}
We remark that this two-step approach has also appeared in \citep[e.g.,][]{menon2018cost,wang2019repairing} for designing ``fair'' classifiers.

\section{Splitting in Practice}
\label{sec::split_practice}

So far we have studied the benefit-of-splitting from an information-theoretic view as we assume the underlying data distribution is known and do not restrict the space of potential classifiers. In this section, we study the effect of splitting classifiers from a more practical perspective. First, we restrict the classifiers over a hypothesis class (e.g., logistic regressions) and analyze the hypothesis class dependent splitting. Second, we consider splitting classifiers in a finite sample regime and study the sample limited splitting.

\subsection{Hypothesis Class Dependent Splitting}
\label{sec::hyp_dep_split}

We restrict both group-blind and split classifiers over the same hypothesis class and introduce a hypothesis class dependent benefit-of-splitting for quantifying the loss reduction by splitting classifiers. 
\begin{defn}
\label{defn::hypothesis_splitting}
For a fixed probability distribution $P_{X,Y|S=s}$ with $s\in \{0,1\}$ and a given hypothesis class $\mathcal{H}$, the $\mathcal{H}$-benefit-of-splitting is defined as 
\begin{align}
    &\epsilon_{\text{split}}^{\mathcal{H}} \defined \inf_{h\in \mathcal{H}}\max_{s\in \{0,1\}} \EE{\lvert h(X) - y_s(X)\rvert \mid S=s} - \max_{s\in \{0,1\}}\inf_{h\in \mathcal{H}} \EE{\lvert h(X) - y_s(X)\rvert \mid S=s}.
\end{align}
\end{defn}
Clearly, the $\mathcal{H}$-benefit-of-splitting maintains the non-maleficence principle $\epsilon_{\text{split}}^{\mathcal{H}} \geq 0$, i.e., given sufficient samples, splitting classifiers will never diminish model accuracy compared to using a group-blind classifier. Next, we provide upper and lower bounds for $\epsilon_{\text{split}}^{\mathcal{H}}$ in order to understand when splitting classifiers brings the most benefit. As before, these bounds rely on three major factors: (i) disagreement between optimal (split) classifiers; (ii) similarity between unlabeled distributions; and (iii) approximation error defined as the smallest loss achieved by split classifiers. In particular, we assume that the last factor is small. This is a common assumption in, e.g., the domain adaptation literature \citep{ben2010theory} since when the hypothesis class is complex enough, this term will be negligible. Furthermore, one central notion of fairness we follow is non-maleficence (i.e., classifiers should avoid the causation of harm on any group). When the approximation error is large, neither group-blind classifiers nor splitting classifiers are accurate and ``harm'' is inevitable. Hence, one should change the hypothesis class first instead of splitting.

\begin{thm}
\label{thm::H_dep_split_bounds}
Let $h_s^*$ be an optimal classifier for group $s\in \{0,1\}$:
\begin{align*}
    h_s^* \in \argmin_{h\in \mathcal{H}} \EE{\lvert h(X)-y_s(X)\rvert \mid S=s}.
\end{align*}
Then we have the following upper and lower bounds for the $\mathcal{H}$-benefit-of-splitting
\begin{align*}
    &\epsilon_{\text{split}}^{\mathcal{H}}
    \leq \min_{s\in \{0,1\}} \EE{\lvert h_1^*(X) - h_0^*(X)\rvert \mid S=s}\\
    &\epsilon_{\text{split}}^{\mathcal{H}}
    \geq \frac{1}{2} \max_{s\in \{0,1\}} \EE{\lvert h_1^*(X) - h_0^*(X)\rvert \mid S=s} - \TV(P_0\| P_1) - \frac{3\sum_{s\in \{0,1\}}\EE{\lvert h_s^*(X) - y_s(X)\rvert \mid S=s}}{2}.
\end{align*}
\end{thm}
\begin{proof}
See Appendix~\ref{append::H_dep_split_bounds}.
\end{proof}
Analogous to our discussions in Section~\ref{sec::acc_imp_split}, the bounds in Theorem~\ref{thm::H_dep_split_bounds} delineate a taxonomy of splitting when both group-blind and split classifiers are restricted over the same hypothesis class: splitting classifiers does not bring much benefit when two groups have similar optimal classifiers; splitting classifiers benefits the most when two groups have similar unlabeled distributions and different optimal classifiers. We further demonstrate this taxonomy of splitting and show how these factors influence the effect of splitting through numerical experiments in Section~\ref{sec::real_world_data}.

In contrast to the upper bound for $\epsilon_{\text{split}}$ (see Theorem~\ref{thm::ben_split_l12_loss_bounds}), the upper bound for $\epsilon_{\text{split}}^{\mathcal{H}}$ does not involve the similarity between the unlabeled distributions. Consequently, when the optimal classifiers are different and the unlabeled distributions are different as well, it is unclear how much benefit splitting classifiers brings. We provide the following example which shows that different hypothesis classes may result in largely different values of the $\mathcal{H}$-benefit-of-splitting. Hence, one must study the effect of splitting on a case-by-case basis for different hypothesis classes.
\begin{example}
\label{ex::split_dep_hyp}
Let two groups' unlabeled distributions and labeling functions be $P_0 \sim \mathcal{N}(-\mu, 1)$, $y_0(x) = \indicator{[x > -\mu]}$ and $P_1 \sim \mathcal{N}(\mu, 1)$, $y_1(x) = \indicator{[x < \mu]}$, respectively. As $\mu$ grows larger, the distance between the unlabeled distributions $P_{0}$ and $P_{1}$ increases (i.e., $\TV(P_{0}\|P_{1})\to 1$ as $\mu\to \infty$). Now we consider the following two hypothesis classes:
\begin{itemize}
    \item $\mathcal{H}_{\text{threshold}}$ is the class of threshold functions over $\Reals$: $\indicator{[x>a]}$ or $\indicator{[x<b]}$.
    \item $\mathcal{H}_{\text{interval}}$ is the class of intervals over $\Reals$: $\indicator{[x\in(a,b)]}$. 
\end{itemize}
Here, $a,b$ are allowed to be $-\infty$ and $+\infty$, respectively. In both cases, the labeling functions are included in the hypothesis classes and, hence, are optimal classifiers. The disagreement between these optimal classifiers is at least $1/2$:
\begin{align*}
    \EE{\lvert y_1(X) - y_0(X)\rvert \mid S=s} \geq 1/2, \quad s\in \{0,1\}.
\end{align*}
The benefit-of-splitting under $\mathcal{H}_{\text{threshold}}$ is $1/2$ as any group-blind classifier incurs at least $1/2$ loss on the disadvantaged group. On the other hand, as $\mu$ becomes larger, the benefit-of-splitting under $\mathcal{H}_{\text{interval}}$ is nearly $0$ since a group-blind classifier with the form $h^*(x) = \indicator{[x\in(-\mu,\mu)]}$ can achieve almost perfect accuracy. 
\end{example}
The previous example shows that using a threshold function as a group-blind classifier will always incur an inevitable accuracy trade-off between two groups. On the other hand, if we enrich the hypothesis class to include interval functions, this trade-off can be reconciled. Motivated by this observation, when two groups have different unlabeled distributions and different labeling functions, we conjecture that the $\mathcal{H}$-benefit-of-splitting is determined by the ``richness'' of the hypothesis class: a more complex hypothesis class can produce a group-blind classifier which mimics the labeling function of each group in the region they lie in, and, hence, this classifier guarantees high accuracy for both groups. We formalize this intuition through the example of feedforward neural networks. 
Recall that a sigmoidal function \citep{barron1993universal} (e.g., logistic function) $S:\Reals \to \Reals$ is a bounded measurable function which satisfies $S(z)\to 1$ as $z\to +\infty$ and $S(z)\to 0$ as $z\to -\infty$. The hypothesis class associated to feedforward neural networks with one layer of sigmoidal functions has the form
\begin{align}
\label{eq::hyp_one_lay_sig}
    \mathcal{H}
    = \left\{\sum_{i=1}^k c_i S(a_i\cdot x + b_i) + c_0 \mid a_i \in \Reals^d, b_i,c_i \in \Reals \right\}.
\end{align}
In this case, Barron's approximation bounds~\citep{barron1993universal} guarantee that these neural networks can approximate a large class of functions reliably. 

\begin{prop}
\label{prop::NN_Barron_ben_split_ub}
Consider the hypothesis class $\mathcal{H}$ in \eqref{eq::hyp_one_lay_sig}. If $\mathcal{X}\subset \Reals^d$ is compact, we have 
\begin{align*}
    \epsilon_{\text{split}}^{\mathcal{H}}
    \leq \min_{s\in\{0,1\}} \sqrt{\EE{(y_1(X)-y_0(X))^2\mid S=s}} \cdot \sqrt{1-\TV(P_0\| P_1)} + \frac{2\mathsf{diam}(\mathcal{X})C}{\sqrt{k}},
\end{align*}
where $\mathsf{diam}(\mathcal{X}) = \sup_{x,x'\in \mathcal{X}} \|x-x'\|_2$, 
\begin{align}
\label{eq::FI}
    h^*(x)
    \defined \frac{y_0(x)\dif P_0(x) + y_1(x) \dif P_1(x)}{\dif P_0(x) + \dif P_1(x)}
    = \int_{\Reals^d} \exp(iw x) \tilde{h^*}(w) \dif w 
\end{align}
for some complex-valued function $\tilde{h^*}$, and $C\defined \int_{\Reals^d} \|w\|_2 |\tilde{h^*}(w)|\dif w$.
\end{prop}
\begin{proof}
See Appendix~\ref{append::NN_Barron}.
\end{proof}
\begin{rem}
The condition in \eqref{eq::FI} goes back to the seminal work of Barron \citep{barron1993universal}. By the Fourier inversion theorem, if both $h^*$ and its Fourier transform are integrable, this condition is satisfied. Further situations where \eqref{eq::FI} holds are discussed in \citep[Section~IX]{barron1993universal}.
\end{rem}

In contrast to Theorem~\ref{thm::H_dep_split_bounds}, the upper bound for the $\mathcal{H}$-benefit-of-splitting above involves the similarity between the unlabeled distributions (i.e., $\TV(P_0\|P_1)$) at the cost of having an additional term which is inversely proportional to the hypothesis class complexity. The intuition behind our proof is that if a data scientist is able to train a neural network with enough neurons, a group-blind classifier is capable of guaranteeing high accuracy for both groups when their unlabeled distributions are different. Consequently, there is no much room for accuracy improvement by splitting classifiers.

\subsection{Comparison with the Cost-of-Coupling}
\label{subsec:comparison}

We compare our notion of the $\mathcal{H}$-benefit-of-splitting with the cost-of-coupling introduced by Dwork~\etal \citep{dwork2018decoupled}. We first illustrate the difference between group blind, coupled, split classifiers through the example of logistic regressions: 
\begin{itemize}
    \item a group blind classifier never uses a sensitive attribute as an input: $h(x) = \mathsf{logistic}(w^Tx)$; 
    \item a coupled classifier uses a sensitive attribute while sharing other parameters: $h(s,x) = \mathsf{logistic}(w^Tx + w_0s)$; 
    \item split classifiers are a set of classifiers applied to each separate group: $h_s(x)=\mathsf{logistic}(w_s^Tx)$.
\end{itemize}
Now we recast the definition of the cost-of-coupling \citep{dwork2018decoupled} using our notation.
\begin{defn}[\citep{dwork2018decoupled}]
Let $\mathcal{H}_{\text{C}}$ be a hypothesis class which contains coupled classifiers from a finite set $\mathcal{S}\times \mathcal{X}$ to $[0,1]$. For a given loss function $\ell(\cdot,\cdot)$, the cost-of-coupling is defined as 
\begin{align*}
    \max_{P_{S,X,Y}} \Big\{ \min_{h\in \mathcal{H}_{\text{C}}}L(h) - \min_{\substack{h_s \in \mathcal{H}_{\text{C}}\\ \text{for } s\in \mathcal{S}} }L(\{h_s\}_{s\in \mathcal{S}})\Big\},
\end{align*}
where the maximum is over all distributions on $\mathcal{S}\times \mathcal{X} \times \{0,1\}$ and $L(h) \defined \EE{\ell(Y, h(S,X))}$, $L(\{h_s\}_{s\in \mathcal{S}}) \defined \EE{\ell(Y, h_S(S,X))}$.
\end{defn}

There are two important differences between the $\mathcal{H}$-benefit-of-splitting (see Definition~\ref{defn::hypothesis_splitting}) and the cost-of-coupling \citep{dwork2018decoupled}. 
First, our notion quantifies the gain in accuracy by using split classifiers rather than a group-blind classifier. In contrast, the cost-of-coupling compares coupled classifiers with split classifiers which both take a sensitive attribute as an input. 
Second, the cost-of-coupling is a worst-case quantity as it maximizes over all distributions. 
By allowing our notion to rely on the data distribution, Definition~\ref{defn::hypothesis_splitting} captures more intricate scenarios for characterizing the benefit of splitting classifiers. 
Furthermore, by taking the maximum over all distributions, we recover an analogous result of Theorem~2 in \citep{dwork2018decoupled}.

\begin{cor}
\label{cor::simp_mod_tradeoff}
There exists a probability distribution $Q_{S,X,Y}$ whose $\mathcal{H}$-benefit-of-splitting is at least $1/2$ under 
\begin{enumerate}
    \item Linear predictors: $\mathcal{H} = \{\indicator{[w^Tx\geq0]} \mid w\in \Reals^d\}$;
    \item Decision trees: $\mathcal{H}$ is the set of binary decision trees.
\end{enumerate}
Furthermore, under this hypothetical distribution $Q_{S,X,Y}$, no matter which group-blind classifier $h\in \mathcal{H}$ is used, there is always a group $s\in \{0,1\}$ such that $\EE{\lvert h(X)-y_s(X)\rvert \mid S=s} \geq 1/2$.
\end{cor}
The proof technique used for this corollary can be extended to many other models (e.g., kernel methods or neural networks) and we defer its proof to Appendix~\ref{append::H_ben_split_geq_1/2}.

\subsection{Sample Limited Splitting}

Consider the following scenario. A data scientist has access to finitely many samples and she/he solves an empirical risk optimization in order to obtain an optimal group-blind classifier or a set of optimal split classifiers. When these classifiers are deployed on new fresh samples, a natural question is whether the optimal split classifiers still outperform the group-blind classifier. We introduce the sample-limited-splitting which quantifies the effect of splitting classifiers within this finite sample regime. 

\begin{defn}
\label{defn::emp_ben_split}
For a given hypothesis class $\mathcal{H}$ and $n_s$ i.i.d. samples $\{(x_{s,i},y_{s,i})\}_{i=1}^{n_s}$ from group $s\in \{0,1\}$, let $\hat{h}^*$ and $\{\hat{h}_s^*\}_{s\in \{0,1\}}$ be optimal group-blind and split classifiers for the empirical $\ell_1$ loss, respectively:
\begin{align}
    &\hat{h}^* \in \argmin_{h\in \mathcal{H}}\max_{s\in\{0,1\}} \frac{\sum_{i=1}^{n_s} |h(x_{s,i}) - y_{s,i}|}{n_s}, \label{eq::opt_em_gb}\\
    &\hat{h}_s^* \in \argmin_{h\in \mathcal{H}} \frac{\sum_{i=1}^{n_s} |h(x_{s,i}) - y_{s,i}|}{n_s}, \quad s\in \{0,1\}.\label{eq::opt_em_dec}
\end{align}
The sample-limited-splitting is defined as
\begin{align}
    \hat{\epsilon}_{\text{split}} \defined \max_{s\in \{0,1\}} \EE{\lvert \hat{h}^*(X) - y_s(X) \rvert \mid S=s} - \max_{s\in \{0,1\}} \EE{\lvert \hat{h}^*_s(X) - y_s(X)\rvert \mid S=s}.
\end{align}
\end{defn}
Unlike the benefit-of-splitting or the $\mathcal{H}$-benefit-of-splitting, the sample-limited-splitting is not necessarily non-negative. In other words, with limited amount of samples available, splitting classifiers may not improve accuracy for both groups. In what follows, we provide data-dependent upper and lower bounds for the sample-limited-splitting in order to understand the effect of splitting classifiers in the finite sample regime. 
\begin{thm}
\label{thm::samp_comp}
Let $\mathcal{H}$ be a hypothesis class from $\mathcal{X}$ to $\{0,1\}$ with VC dimension $D$. If $\hat{h}_s^*$ is a minimizer of the empirical $\ell_1$ loss $\sum_{i=1}^{n_s} |h(x_{s,i}) - y_{s,i}|/n_s$ computed via $n_s$ i.i.d. samples $\{(x_{s,i},y_{s,i})\}_{i=1}^{n_s}$, then, with probability at least $1-\delta$, 
\begin{align}
    \hat{\epsilon}_{\text{split}}
    \leq & \min_{s\in \{0,1\}} \frac{\sum_{i=1}^{n_s}|\hat{h}_1^*(x_{s,i}) - \hat{h}_0^*(x_{s,i})|}{n_s} + \Omega,\\
    \begin{split}
    \label{eq::lb_em_ben_split}
    \hat{\epsilon}_{\text{split}}
    \geq & \frac{1}{2}\max_{s\in \{0,1\}} \frac{\sum_{i=1}^{n_s}|\hat{h}_1^*(x_{s,i}) - \hat{h}_0^*(x_{s,i})|}{n_s} - \TV(\hat{P}_0\| \hat{P}_1) - 3\lambda - \Omega,
    \end{split}
\end{align}
where $\hat{P}_s$ is the empirical unlabeled distribution and 
\begin{align*}
    \lambda 
    &\defined \frac{1}{2}\left(\frac{\sum_{i=1}^{n_0}|\hat{h}_0^*(x_{0,i}) - y_{0,i}|}{n_0} + \frac{\sum_{i=1}^{n_1} |\hat{h}_1^*(x_{1,i}) - y_{1,i}|}{n_1}\right), \quad
    \Omega 
    \defined 4\max_{s\in \{0,1\}} \sqrt{\frac{2D\log(6n_s) + 2\log(8/\delta)}{n_s}}.
\end{align*}
\end{thm}
\begin{proof}
See Appendix~\ref{append::samp_comp}.
\end{proof}
Here, the term $\lambda$ is the (average) training loss and $\Omega$ is the complexity term, which is approximately $\sqrt{D/\min\{n_0, n_1\}}$. As shown, the upper and lower bounds for $\hat{\epsilon}_{\text{split}}$ rely on four factors. The first three factors, which also appear in our bounds of the $\mathcal{H}$-benefit-of-splitting (see Theorem~\ref{thm::H_dep_split_bounds}), are the disagreement between the (empirically) optimal classifiers, the similarity of the (empirically) unlabeled distributions, and the (empirically) training error. In addition to these factors, our bounds for $\hat{\epsilon}_{\text{split}}$ also depend on the number of samples from each group, especially minority group with less samples, and model complexity (measured by the VC dimension \citep{anthony2009neural}). 
\section{Numerical Experiments}
\label{sec::exp}

We illustrate the theoretical results presented in this paper through experiments. 
In Section~\ref{sec::perf_imp_split}, we presented an algorithm (Algorithm~\ref{alg:descent}) for computing the benefit-of-splitting. In particular, when the data distribution is known, this algorithm provably converges to the exact value of the benefit-of-splitting. To evaluate Algorithm~\ref{alg:descent}, we conduct experiments on a synthetic example where both the data distribution and the values of the benefit-of-splitting are known. 
In Section~\ref{sec::split_practice}, we characterized a taxonomy of splitting when classifiers are restricted over a hypothesis class. We demonstrate this taxonomy of splitting through experiments on 40 real-world datasets.

\subsection{Synthetic Datasets}
\label{subsec::syn_data}

\begin{figure*}[t]
\centering
\includegraphics[width=0.45\linewidth]{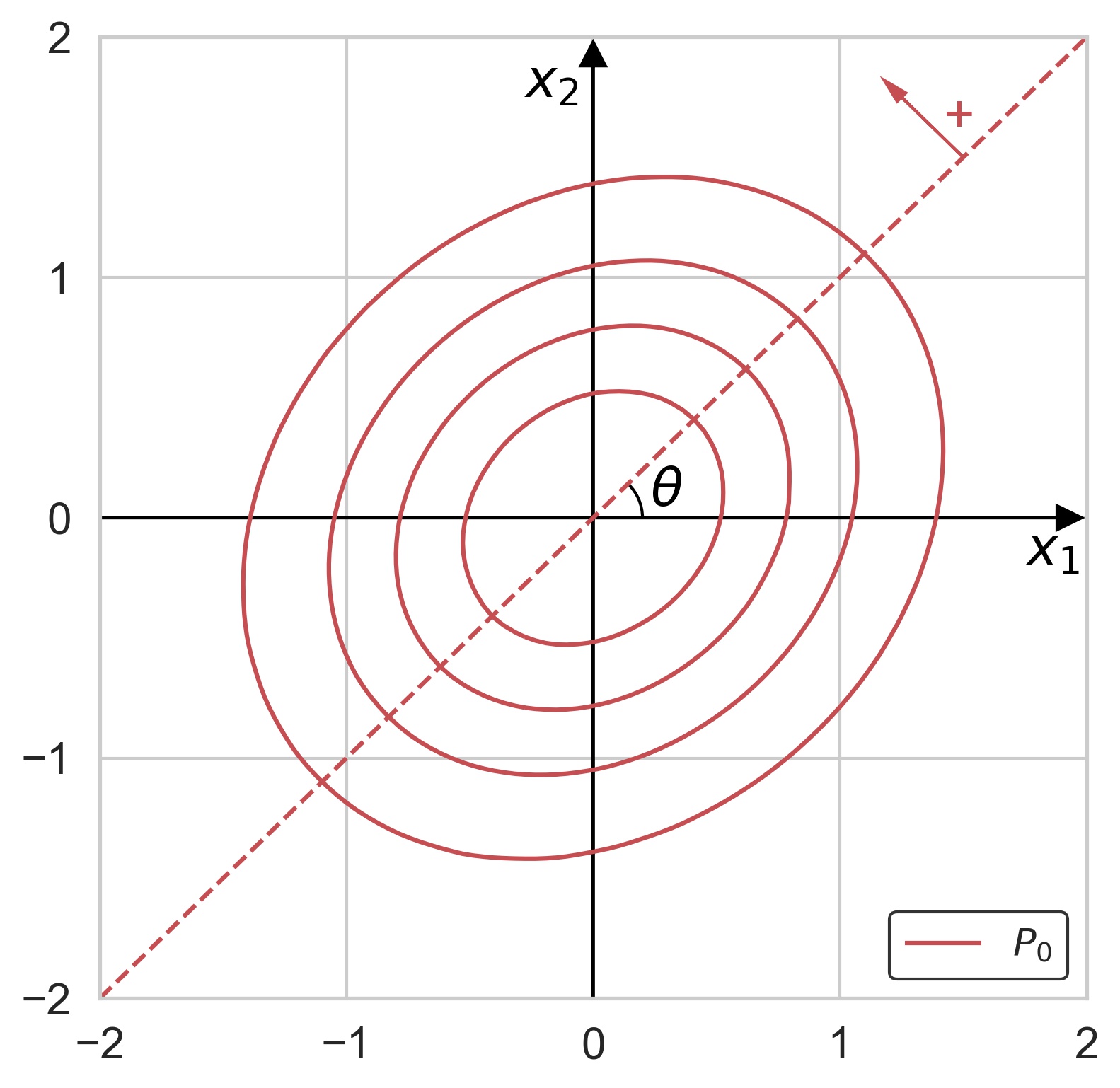}
\includegraphics[width=0.45\linewidth]{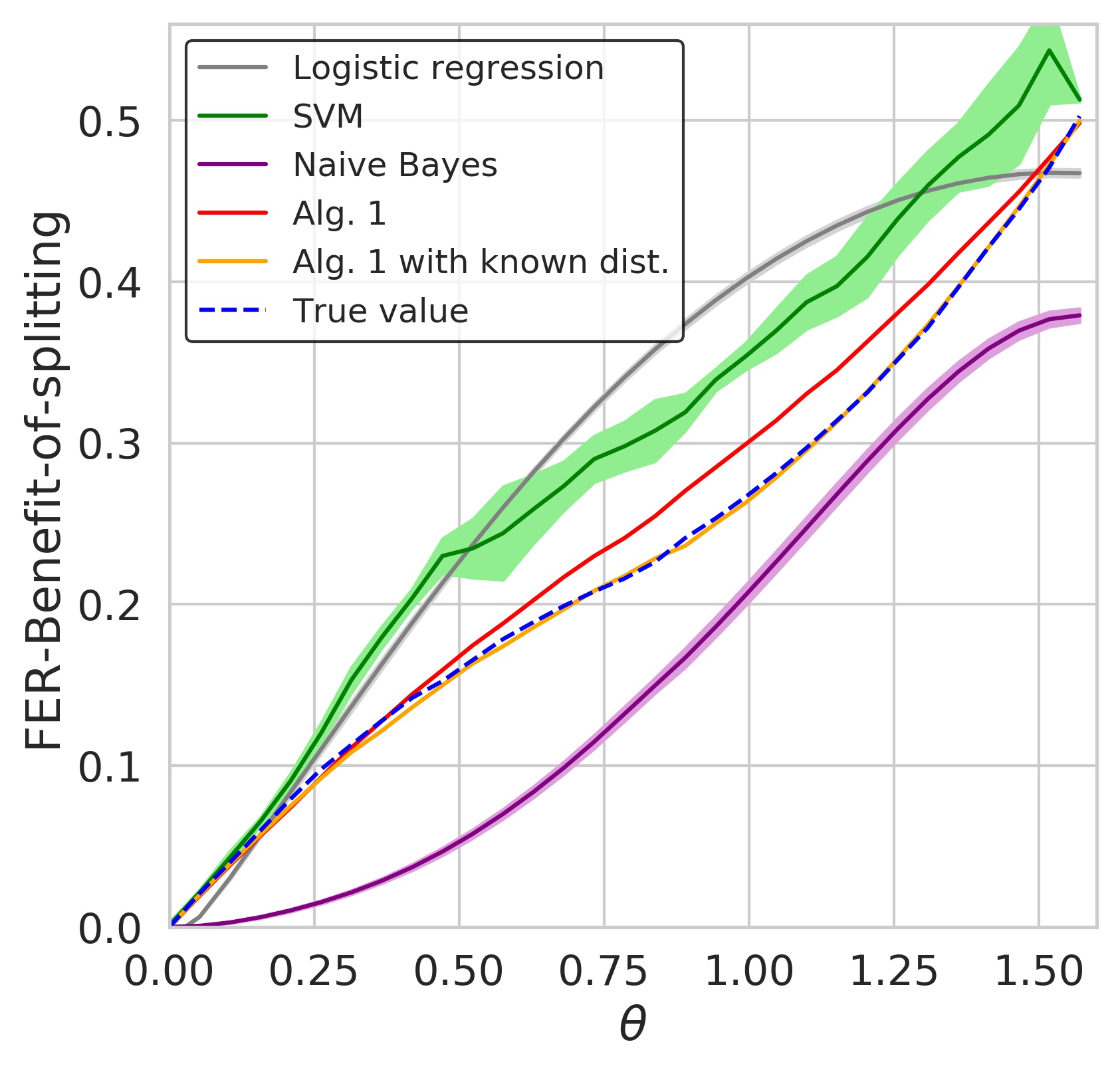}
\caption{\small{We demonstrate the performance of Algorithm~\ref{alg:descent} for computing the FER-benefit-of-splitting $\epsilon_{\text{split},\mathsf{FER}}$ on synthetic datasets. Left: the ellipses are the level sets of the unlabeled distribution $P_0$ and the dash line is the labeling function $y_0$ with a arrow indicating the region where points are labeled as $+$.  Right: $\epsilon_{\text{split},\mathsf{FER}}$ computed by different approaches along with its true values. As shown, when the underlying data distribution is known, the approximation of $\epsilon_{\text{split},\mathsf{FER}}$ produced by Algorithm~\ref{alg:descent} (orange curve) recovers its true values (blue dash curve). When the underlying distribution is unknown, we train binary classifiers and feed them into Algorithm~\ref{alg:descent}. The approximation of $\epsilon_{\text{split},\mathsf{FER}}$ produced by Algorithm~\ref{alg:descent} is depicted as red curve. Finally, we compute $\epsilon_{\text{split},\mathsf{FER}}$ empirically by training optimal group-blind and split classifiers via (i) logistic regression (gray curve), (ii) linear SVM (green curve), (iii) Naive Bayes classifier (purple curve). We use 5-fold cross validation for training these optimal classifiers and plot the standard deviation as shaded region. As shown, the approximations of $\epsilon_{\text{split},\mathsf{FER}}$ produced by Algorithm~\ref{alg:descent} outperform all three empirical approximations of $\epsilon_{\text{split},\mathsf{FER}}$.
}
}
\label{Fig::Gaussian_ben_FER_val}
\end{figure*}

We introduced the FER-benefit-of-splitting $\epsilon_{\text{split},\mathsf{FER}}$ in Section~\ref{subsec::FER_BOS} and proposed an efficient procedure for computing this quantity (Algorithm~\ref{alg:descent}). Here, we validate Algorithm~\ref{alg:descent} through experiments on synthetic datasets. For a fixed parameter $\theta\in [0,\pi/2]$, let two groups' unlabeled distributions be zero-mean Gaussian distributions with different covariance matrices: $P_0 \sim \mathcal{N}\left(\mathbf{0}, \bSigma_0\right)$ and $P_1 \sim \mathcal{N}\left(\mathbf{0}, \bSigma_1\right)$ where
\begin{align*}
    \bSigma_0 = 
    \begin{pmatrix}
        0.5\cos(\theta)^2 + 1 & 0.5\sin(\theta)\cos(\theta)\\
        0.5\sin(\theta)\cos(\theta) & 0.5\sin(\theta)^2 + 1
    \end{pmatrix}
    ,\quad
    \bSigma_1 = 
    \begin{pmatrix}
        0.5\cos(\theta)^2 + 1 & -0.5\sin(\theta)\cos(\theta)\\
        -0.5\sin(\theta)\cos(\theta) & 0.5\sin(\theta)^2 + 1
    \end{pmatrix}
    .
\end{align*}
The distributions $P_0$ and $P_1$ correspond to $\theta$ counterclockwise and clockwise rotation of the Gaussian distribution $\mathcal{N}(\mathbf{0}, \diag(1.5,1))$. Furthermore, let the labeling functions be 
\begin{align*}
    y_0(x) = 
    \begin{cases}
        1 &\text{ if } (-\sin(\theta), \cos(\theta))\cdot x > 0\\
        0 &\text{ otherwise},
    \end{cases}
    \quad
    y_1(x) = 
    \begin{cases}
        1 &\text{ if } (\sin(\theta), \cos(\theta))\cdot x > 0\\
        0 &\text{ otherwise}.
    \end{cases}
\end{align*}
The left-hand side of Figure~\ref{Fig::Gaussian_ben_FER_val}  displays the level sets of $P_0$ as well as its labeling function.

In this synthetic example, $\epsilon_{\text{split},\mathsf{FER}}$ has a closed-form expression: $\epsilon_{\text{split},\mathsf{FER}} = 2\Pr(X\in \mathcal{A}\mid S=0)$ where $\mathcal{A}\defined \{x=(x_1,x_2)\in\Reals^2\mid y_1(x) = 1, x_2<0\}$ (see Appendix~\ref{append::syn_FER_ben_split} for a proof). When $\theta = 0$, two groups share the same unlabeled distribution (i.e., $P_0 = P_1$) and the same labeling function (i.e., $y_0 = y_1$). Hence, there is no benefit of splitting classifiers: $\epsilon_{\text{split},\mathsf{FER}} = 0$. On the other hand, when $\theta = \pi/2$, two groups have the same unlabeled distribution but completely different labeling functions. Splitting classifiers achieves the most benefit: $\epsilon_{\text{split},\mathsf{FER}} = 0.5$.

By varying the values of $\theta$ and drawing 10k samples from each group, we compare the true values of $\epsilon_{\text{split},\mathsf{FER}}$ with the outputs from Algorithm~\ref{alg:descent} as well as other empirical approximations. Recall that Algorithm~\ref{alg:descent} requires a conditional distribution $\Pr(S=s\mid X=x)$ and the labeling functions $y_0$ and $y_1$. Since the conditional distribution and labeling functions are known in this synthetic example, we feed their explicit forms into Algorithm~\ref{alg:descent} for computing $\epsilon_{\text{split},\mathsf{FER}}$ (orange curve in Figure~\ref{Fig::Gaussian_ben_FER_val} Right). 
In practice, the conditional distribution and labeling functions are unknown, so we also train a Naive Bayes classifier \citep{friedman2001elements} to approximate $\Pr(S=s\mid X=x)$ and two linear support-vector machine (SVM) classifiers \citep{friedman2001elements} to approximate the labeling functions. By feeding these binary classifiers into Algorithm~\ref{alg:descent}, another approximation of $\epsilon_{\text{split},\mathsf{FER}}$ is output (red curve in Figure~\ref{Fig::Gaussian_ben_FER_val} Right).
Furthermore, we compute $\epsilon_{\text{split},\mathsf{FER}}$ empirically by training optimal group-blind and split classifiers via logistic regression, linear SVM, or Naive Bayes classifier. Computing the false error rate reduction leads to three empirical approximations of
$\epsilon_{\text{split},\mathsf{FER}}$.

As shown in Figure~\ref{Fig::Gaussian_ben_FER_val}, when Algorithm~\ref{alg:descent} has access to the explicit forms of $\Pr(S=s\mid X=x)$, $y_0$, and $y_1$, it accurately recovers $\epsilon_{\text{split},\mathsf{FER}}$. This is remarkable since even with the knowledge of the underlying distributions, it is unclear how to compute $\epsilon_{\text{split},\mathsf{FER}}$ directly from its definition. 
We also observe that Algorithm~\ref{alg:descent} applied to binary classifiers outputs more accurate approximation of $\epsilon_{\text{split},\mathsf{FER}}$ than the approximations produced by using logistic regression, linear SVM, or Naive Bayes classifier. 

To summarize, we conclude that (i) when the underlying distribution is known, Algorithm~\ref{alg:descent} can produce the precise values of $\epsilon_{\text{split},\mathsf{FER}}$ and has convergence guarantees; (ii) when Algorithm~\ref{alg:descent} is fed with binary classifiers, it produces reliable approximation of $\epsilon_{\text{split},\mathsf{FER}}$; (iii) computing the FER-benefit-of-splitting empirically by training optimal classifiers could incur high approximation errors.

\subsection{Real-world Datasets}
\label{sec::real_world_data}

\begin{figure*}[t]
\centering
\includegraphics[width=0.95\linewidth]{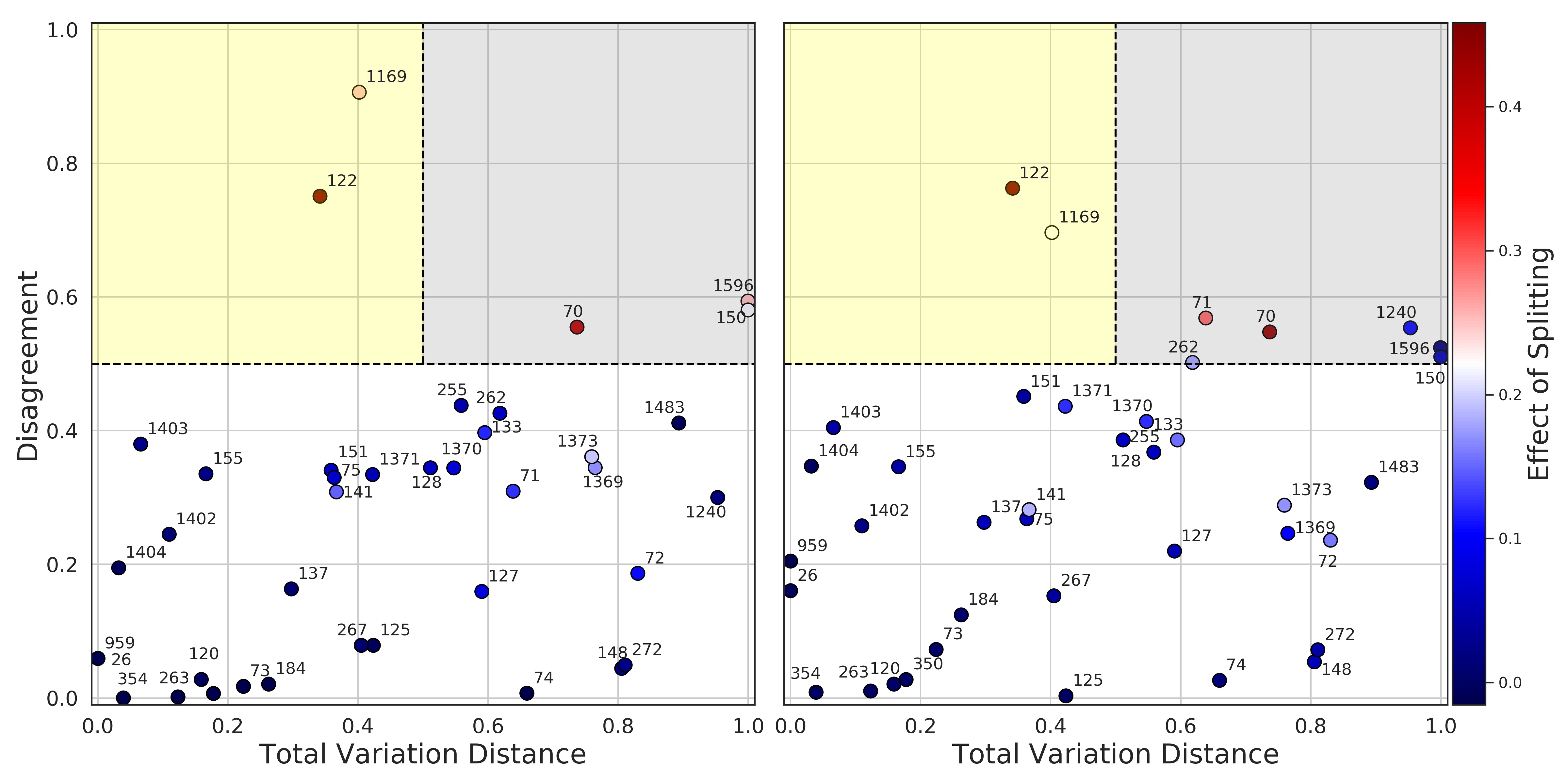}
\caption{\small{We demonstrate how the effect of splitting classifiers is determined by the two factors: disagreement between optimal classifiers (y-axis) and total variation distance between unlabeled distributions (x-axis). 
We restrict both group-blind and split classifiers to logistic regression classifiers (left) or decision tree classifiers (right).
Each dot represents a dataset in OpenML \citep{OpenML2013} with color indicating the effect of splitting classifiers compared to using a group-blind classifier and texts indicating dataset ID. Our upper and lower bounds in Theorem~\ref{thm::H_dep_split_bounds} reveal a taxonomy of splitting where splitting does not bring much benefit (white region); splitting  brings the most benefit (yellow region); or splitting has undetermined effect (grey region).
}
}
\label{Fig::Real_data}
\end{figure*}

In Section~\ref{sec::split_practice}, we analyzed the effect of splitting classifiers when both group-blind and split classifiers are restricted over the same hypothesis class. The bounds in Theorem~\ref{thm::H_dep_split_bounds} reveal two main factors that could determine this effect: disagreement between optimal classifiers and similarity between unlabeled distributions. Here we demonstrate how these two factors influence the effect of splitting through experiments on 40 real-world datasets, collected from OpenML \citep{OpenML2013}.

\paragraph{Setup.}
We preprocess all 40 datasets by adopting the  procedure described in \citep{dwork2018decoupled}. 
All categorical features are transformed into binary by assigning the most frequent object to $1$ and the rest of the objects to $0$. 
The first binary feature is selected as the sensitive attribute and, hence, these datasets are ``semi-synthetic''. 
We truncate the datasets so that each group contains at most 10k data points. In each dataset, there are at least 8k data points per group, minimizing the effect of potential lack of samples per group.

\paragraph{Implementation.}

We obtain optimal split classifiers via training a logistic regression model with the \texttt{LIBLINEAR} solver \citep{fan2008liblinear}, fitting the model by drawing samples from each group. Since an optimal group-blind classifier is a minimizer of $\min_{h\in\mathcal{H}}\max_{w\in [0,1]} w L_0(h) + (1-w)L_1(h)$ where $L_s(h)$ is the loss of a classifier $h$ on group $s\in \{0,1\}$, we solve this optimization approximately by considering its dual formula $\max_{w\in [0,1]} \min_{h\in\mathcal{H}} w L_0(h) + (1-w) L_1(h)$ and use 5-fold cross validation to tune the parameter $w$ therein. Although this procedure of training group-blind classifier needs access to data points' sensitive attribute, it does not violate group-blindness \citep{lipton2018does} because the output classifier does not use the sensitive attribute as an input when deploying on new data. In addition to logistic regressions, we repeat this experiment by training decision tree classifiers with depth 7. The disagreement between optimal classifiers is calculated by applying the optimal split classifiers on each data point and computing the discrepancy. We estimate the total variation distance between unlabeled distributions by applying the procedures introduced in \citep{kanamori2011f} (see Appendix~\ref{append::real_data_TV} for more details).

\paragraph{Result.}

In Figure~\ref{Fig::Real_data}, we illustrate the taxonomy of splitting delineated by our bounds in Theorem~\ref{thm::H_dep_split_bounds}. We restrict the hypothesis class to be logistic regression (Figure~\ref{Fig::Real_data} Left) or to be decision trees with depth 7 (Figure~\ref{Fig::Real_data} Right). Each dot in the figures represents a dataset with its corresponding ID number in the OpenML dataset. The color captures the loss reduction by using the optimal split classifiers compared to deploying the optimal group-blind classifier (red means splitting has more benefit and blue means splitting does not bring much benefit). The location of each dot is determined by the two factors: disagreement between optimal classifiers (y-axis) and total variation distance between unlabeled distributions (x-axis). 
\begin{itemize}
    \item The upper bound in Theorem~\ref{thm::H_dep_split_bounds} indicates that splitting does not bring much benefit when the optimal classifiers are similar. As shown in Figure~\ref{Fig::Real_data}, all datasets which are below the horizontal dash line have small benefit by splitting classifiers (i.e., dots are blue). 
    
    \item The lower bound in Theorem~\ref{thm::H_dep_split_bounds} indicates that splitting benefits model performance when the optimal classifiers are different and the unlabeled distributions are similar. As shown in Figure~\ref{Fig::Real_data}, there are two datasets (ID 122 and 1169) which are in the yellow region and they all achieve large benefit from splitting classifiers. 
    
    \item When both the optimal classifiers and the unlabeled distributions are different, the effect of splitting classifiers can not be determined by the bounds in Theorem~\ref{thm::H_dep_split_bounds}. As shown in Figure~\ref{Fig::Real_data}, the datasets in the grey region could have either large benefit by splitting classifiers or limited benefit. Furthermore, we have conjectured (see Section~\ref{sec::hyp_dep_split}) that in this case a more complex hypothesis class leads to less benefit from splitting classifiers. This is further evidenced in the experiments: when both group-blind and split classifiers are logistic regressions (Figure~\ref{Fig::Real_data} Left), the datasets which are in the grey region all achieve non-trivial benefit by splitting classifiers. In contrast, when decision trees are used (Figure~\ref{Fig::Real_data} Right), there are datasets (e.g., ID 1240) in the grey region which achieve a limited amount of benefit by splitting. 
\end{itemize}

\section{Conclusion and Future Work}
\label{sec:conc}

Split classifiers should only be considered when it is ethical and legal to do so, and when it does not result in harm to any underlying group. Eliminating disparate treatment does not necessarily lead to a  group-fair classifier. On the one hand, a sensitive attribute could correlate with other proxy variables which are used for decision making \citep{hunt2005redlining,wang2019repairing}. On the other hand, the sensitive attribute can be an important feature for the prediction task \citep{kleinberg2018algorithmic,corbett2018measure}. In the latter case,  using a group-blind classifier for achieving treatment parity may lead to an unfavorable accuracy trade-off. 

Motivated by the above discussion, we investigated the following fundamental question: when disparate treatment is allowed, is it beneficial to incorporate the sensitive attribute as an input feature in order to improve a classifier's performance? Due to the bias-variance trade-off, in practice, the answer will depend on the number of training data and the complexity of the hypothesis class. 
In this paper, we focused on an information-theoretic regime where the underlying data distribution is known---or infinitely many data points are available---and the hypothesis class is unrestricted. To evaluate the potential gain in average performance from allowing a classifier to exhibit disparate treatment, we compare split classifiers with group-blind classifiers and characterize precise conditions where splitting classifiers achieves the most benefit. Our results show that---in this narrow information-theoretic regime---splitting classifiers follows the non-maleficence principle and allows a data scientist to deploy more accurate and suitable models for each group.

There are two open questions that deserve further exploration. First, our bounds indicate that the difference in underlying data distributions between groups, the number of samples, and the hypothesis class can all influence the effect of splitting classifiers. Nonetheless, we believe that there are more factors that play an important role in determining such an effect. For example, a group-blind classifier may perform worse on minority groups due to unbalanced samples in the training process and using split classifiers could potentially reconcile this issue. In a similar vein, the lack of sample diversity (i.e., training datasets do not contain enough samples from minority groups) could affect the performance and generalization of ML models for minority groups. Hence, it is crucial to characterize the impact of sample size and diversity on detecting and reducing discrimination. Second, we introduce the sample-limited-splitting $\hat{\epsilon}_{\text{split}}$ for quantifying the effect of splitting classifiers in the finite sample regime and provide its upper and lower bounds. It would be interesting to characterize precise conditions under which $\hat{\epsilon}_{\text{split}} \geq 0$ (or $\hat{\epsilon}_{\text{split}} \leq 0$).

\section*{Acknowledgements}
The authors would like to thank Dr. Berk Ustun for his valuable input. The authors would also like to thank the anonymous reviewers and the associate editor for their constructive feedback.

%\clearpage
\bibliographystyle{IEEEtran}
\bibliography{references}

\appendices
\clearpage

\section{Examples of $f$-divergence}
\label{append::f-divergence}

We recall some examples of $f$-divergence \citep{csiszar1967information} here. 
\begin{itemize}
    \item KL-divergence \citep{kullback1951information}: $f(x) = x\log(x)$,
    \begin{align}
        \KL(P\|Q) = \int \log\left(\tfrac{\dif P}{\dif Q}\right) \dif P.
    \end{align}
    \item Total variation distance: $f(x)=|x-1|/2$,
    \begin{align}
        \TV(P\|Q) = \frac{1}{2} \int \left|\tfrac{\dif P}{\dif Q} - 1\right| \dif Q.
    \end{align}
    \item Chi-square divergence \citep{pearson1900x}: $f(x)=(x-1)^2$ or $f(x) = x^2-1$,
    \begin{align}
        \chidiv(P\|Q) 
        = \int \left(\tfrac{\dif P}{\dif Q} - 1\right)^2 \dif Q 
        = \int \left(\tfrac{\dif P}{\dif Q}\right)^2 \dif Q - 1.
    \end{align}
    \item Jensen-Shannon divergence \citep{lin1991divergence}: $f(x) = x\log(x)/2 - (1+x)\log((1+x)/2)/2$,
    \begin{align}
        \mathsf{JS}(P\|Q) = \frac{1}{2}\KL\left(P\|\tfrac{P+Q}{2}\right) + \frac{1}{2}\KL\left(Q\|\tfrac{P+Q}{2}\right).
    \end{align}
    Note that the Jensen-Shannon divergence is defined in a general form in \citep{lin1991divergence} for $\omega\in [0,1]$
    \begin{align}
    \label{eq::defn_JS_general}
        \mathsf{JS}_{\omega}(P\|Q) = \omega\KL\left(P\|\omega P + (1-\omega) Q\right) + (1-\omega)\KL\left(Q\|\omega P + (1-\omega) Q \right).
    \end{align}
    \item $E_{\gamma}$-divergence (also called hockey-stick divergence) \citep{sason2016f,polyanskiy2010channel,sharma2013fundamental,liu2016Egamma}: $f(x) = (x-\gamma)_{+}$ for $\gamma\geq 1$ where $(a)_{+} \defined \max\{a,0\}$,
    \begin{align}
        E_{\gamma}(P\|Q) = \int \left(\tfrac{\dif P}{\dif Q} - \gamma \right)_{+} \dif Q.
    \end{align}
    \item DeGroot statistical information \citep{degroot1962uncertainty} of order $p$: $f(x) = \min\{p,1-p\} - \min\{p,1-px\}$ for $p\in (0,1)$,
    \begin{align}
        \mathcal{I}_p(P\|Q) = \min\{p,1-p\} - \int \min\left\{p, 1-p\tfrac{\dif P}{\dif Q}\right\}\dif Q.
    \end{align}
    \item Marton's divergence \citep{marton1996measure}: $f(x) = (x-1)^2 \indicator{[x<1]}$,
    \begin{align}
        d_2(P\|Q)^2 
        = \inf \EE{\Pr(X\neq Y \mid Y)^2}
        = \int \left(\tfrac{\dif P}{\dif Q} -1 \right)^2 \indicator{[\tfrac{\dif P}{\dif Q} < 1]} \dif Q,
    \end{align}
    where the infimum is taken over all couplings, i.e., joint distributions $P_{X,Y}$ which have marginals $P_X = P$ and $P_Y = Q$, respectively. 
\end{itemize}
We refer the readers to \citep{sason2016f,raginsky2016strong} for more examples of $f$-divergence and their properties.

\section{Proofs for Section~\ref{sec::acc_imp_split}}

\subsection{Proof of Lemma~\ref{lem::ben_split_l12_loss_equivalent}}
\label{appendix::split_acc_loss_equivalent}

\begin{proof}
We first introduce a (measurable) loss function $\ell:[0,1]\times [0,1]\to \Reals^+\cup \{\infty\}$ and assume that this loss function satisfies: (i) $\ell(a,a) = 0$ for any $a\in [0,1]$ and (ii) for any $a\in[0,1]$, $\ell(a,\cdot)$ is convex and continuous. The benefit-of-splitting in Definition~\ref{defn::ben_split_l1_2} can be written as
\begin{align}
\label{eq::bene_split_accu_based}
    \inf_{h:\mathcal{X}\to [0,1]}\max_{s\in\{0,1\}} \EE{\ell(y_s(X), h(X)) \mid S=s} - \max_{s\in\{0,1\}} \inf_{h:\mathcal{X}\to [0,1]}\EE{\ell(y_s(X), h(X)) \mid S=s}.
\end{align}
By taking the $\ell_1$ loss $\ell(a,b) = |a-b|$, $\ell_2$ loss $\ell(a,b) = (a-b)^2$, and KL loss $\ell(a,b) = \KL(a\|b) \defined a\log(a/b) + (1-a)\log((1-a)/(1-b))$, respectively, the above quantity becomes $\epsilon_{\text{split},1}$, $\epsilon_{\text{split},2}$, and $\epsilon_{\text{split},\mathsf{KL}}$. These loss functions all satisfy our above two assumptions. In particular, by our first assumption, one can choose $h(x) = y_s(x)$ which leads to 
\begin{align}
\label{eq::max_inf_loss_0}
    \max_{s\in\{0,1\}} \inf_{h:\mathcal{X}\to [0,1]}\EE{\ell(y_s(X), h(X)) \mid S=s} = 0.
\end{align}
Hence, the problem remains providing equivalent expression for the inf-max term
\begin{align}
    &\inf_{h:\mathcal{X}\to [0,1]}\max_{s\in\{0,1\}} \EE{\ell(y_s(X), h(X)) \mid S=s} \nonumber\\
    &= \inf_{h:\mathcal{X}\to [0,1]}\sup_{\omega\in[0,1]} \omega\cdot\EE{\ell(y_0(X), h(X)) \mid S=0} + (1-\omega)\cdot\EE{\ell(y_1(X), h(X)) \mid S=1}. \label{eq::inf_sup_lem_loss_bos}
\end{align}
Next, we use Ky Fan's min-max theorem~\citep{fan1953minimax} (see Lemma~\ref{lem::KyFan}) to swap the positions of infimum and supremum in \eqref{eq::inf_sup_lem_loss_bos}. We start with verifying the assumptions in Ky Fan's min-max theorem. We denote the set of all measurable functions from $\mathcal{X}$ to $[0,1]$ by $\mathcal{L}(\mathcal{X}\to [0,1])$ and introduce a function $F:[0,1]\times \mathcal{L}(\mathcal{X}\to [0,1]) \to \Reals$
\begin{align*}
    F(\omega, h)\defined \omega\cdot\EE{\ell(y_0(X), h(X)) \mid S=0} + (1-\omega)\cdot\EE{\ell(y_1(X), h(X)) \mid S=1}.
\end{align*}
For every fixed $h\in \mathcal{L}(\mathcal{X}\to [0,1])$, $F(\cdot, h)$ is a linear function. Consequently, $F(\cdot, h)$ is upper semicontinuous and $F$ is concave-like on $[0,1]$. Furthermore, for any $h_1,h_1\in \mathcal{L}(\mathcal{X}\to [0,1])$, $\lambda \in [0,1]$, and $\omega \in [0,1]$, we have $\lambda h_1 + (1-\lambda)h_2 \in \mathcal{L}(\mathcal{X}\to [0,1])$ and 
\begin{align*}
    F(\omega, \lambda h_1 + (1-\lambda)h_2) \leq \lambda F(\omega, h_1) + (1-\lambda) F(\omega, h_2)
\end{align*}
by the convexity of $\ell(a,\cdot)$ for any $a\in [0,1]$. Hence, $F$ is convex-like on $\mathcal{L}(\mathcal{X}\to [0,1])$. Therefore, by Ky Fan's min-max theorem, \eqref{eq::inf_sup_lem_loss_bos} is equal to
\begin{align}
\label{eq::sup_inf_abst_loss_linear_comb}
    \sup_{\omega\in[0,1]} \inf_{h:\mathcal{X}\to [0,1]} \omega\cdot\EE{\ell(y_0(X), h(X)) \mid S=0} + (1-\omega)\cdot\EE{\ell(y_1(X), h(X)) \mid S=1}.
\end{align}
Now we take any probability distribution $P$ over $\mathcal{X}$ such that $P_0$ and $P_1$ are absolutely continuous with respect to $P$. For example, one can simply choose $\dif P = (\dif P_0 + \dif P_1)/2$. Then \eqref{eq::sup_inf_abst_loss_linear_comb} can be written as 
\begin{align}
\label{eq::sup_inf_abst_loss_int}
    \sup_{\omega\in[0,1]} \inf_{h:\mathcal{X}\to [0,1]} \int \left(\omega \cdot \ell(y_0(x), h(x)) \tfrac{\dif P_0(x)}{\dif P(x)} + (1-\omega)\cdot \ell(y_1(x), h(x)) \tfrac{\dif P_1(x)}{\dif P(x)} \right) \dif P(x).
\end{align}
Next, we prove that the infimum and the integer in \eqref{eq::sup_inf_abst_loss_int} can be interchanged. For a fixed $\omega \in [0,1]$, we introduce a function $f:\mathcal{X}\times[0,1]\to \Reals$
\begin{align*}
    f(x,\bar{h}) 
    \defined \omega \cdot \ell(y_0(x), \bar{h}) \tfrac{\dif P_0(x)}{\dif P(x)} + (1-\omega)\cdot \ell(y_1(x), \bar{h}) \tfrac{\dif P_1(x)}{\dif P(x)}
\end{align*}
and aim at proving
\begin{align}
\label{eq::inf_int_interchange}
    \inf_{h:\mathcal{X}\to [0,1]} \int f(x,h(x)) \dif P(x)
    = \int \inf_{\bar{h}\in [0,1]} f(x,\bar{h}) \dif P(x).
\end{align}
Since $f(\cdot,\bar{h})$ is measurable and $f(x,\cdot)$ is continuous, $f$ is a Carath\'eodory function \citep[see Section 4.10 in][]{guide2006infinite}. Hence, by the measurable maximum theorem \citep[see Theorem~18.19 in][]{guide2006infinite}, the mapping
\begin{align*}
    x \to \inf_{\bar{h}\in [0,1]} f(x,\bar{h})
\end{align*}
is measurable and the argmin correspondence (i.e., set-valued function)
\begin{align*}
    \mathcal{H}^*(x) \defined \left\{\bar{h}^* \in [0,1] \mid f(x,\bar{h}^*) = \inf_{\bar{h}\in [0,1]} f(x,\bar{h}) \right\}
\end{align*}
is also measurable and admits a measurable selector. We denote this selector by $h^*:\mathcal{X}\to [0,1]$ and, by definition, it satisfies $h^*(x)\in \mathcal{H}^*(x)$ for all $x\in\mathcal{X}$. Now we are ready to prove \eqref{eq::inf_int_interchange}. One direction $\text{LHS} \geq \text{RHS}$ can be obtained directly since for any $h:\mathcal{X}\to [0,1]$ 
\begin{align*}
    \int f(x,h(x)) \dif P(x)
    \geq \int \inf_{\bar{h}\in [0,1]} f(x,\bar{h}) \dif P(x).
\end{align*}
By the definition of $h^*(x)$, 
\begin{align*}
    \text{RHS} = \int f(x,h^*(x)) \dif P(x) 
    \geq \inf_{h:\mathcal{X}\to [0,1]} \int f(x,h(x)) \dif P(x) 
    = \text{LHS}.
\end{align*}
Therefore, the equality in \eqref{eq::inf_int_interchange} holds and \eqref{eq::sup_inf_abst_loss_int} becomes
\begin{align}
\label{eq::sup_opt_class_plugin}
    \sup_{\omega\in[0,1]} \int \left(\omega \cdot \ell(y_0(x), h^*(x)) \tfrac{\dif P_0(x)}{\dif P(x)} + (1-\omega)\cdot \ell(y_1(x), h^*(x)) \tfrac{\dif P_1(x)}{\dif P(x)} \right) \dif P(x).
\end{align}
Hence, our last step is to compute the function $h^*$ for the loss functions of interest. If the loss function is $\ell_1$, then 
\begin{align*}
    \argmin_{\bar{h}\in[0,1]} f(x,\bar{h})
    = \argmin_{\bar{h}\in[0,1]} \left\{ \omega \frac{\dif P_0(x)}{\dif P(x)} \cdot |\bar{h}-y_0(x)| + (1- \omega) \frac{\dif P_1(x)}{\dif P(x)} \cdot |\bar{h}-y_1(x)| \right\}.
\end{align*}
For a fixed $\omega\in [0,1]$, the optimal classifier is
\begin{align*}
    h^*(x) = 
    \begin{cases}
        y_0(x) &\text{ if } \frac{\dif P_0(x)}{\dif P_1(x)} \geq \frac{1-\omega}{\omega} \\
        y_1(x) &\text{ otherwise}.
    \end{cases}
\end{align*}
By substituting the optimal classifier and $\ell_1$ loss into \eqref{eq::sup_opt_class_plugin}, we get the desired equivalent expression of $\epsilon_{\text{split},1}$:
\begin{align*}
    \sup_{\omega \in [0,1]} (1-\omega) \int_{\mathcal{A}_{\omega}} |y_1(x)-y_0(x)| \dif P_1(x) + \omega \int_{\mathcal{A}_{\omega}^{c}} |y_1(x)-y_0(x)|\dif P_0(x),
\end{align*}
where $\mathcal{A}_{\omega} \defined \left\{x \mid \frac{\dif P_0(x)}{\dif P_1(x)} \geq \frac{1-\omega}{\omega} \right\}$. Similarly, when $\ell_2$ loss is used, the optimal classifier becomes
\begin{align}
\label{eq::opt_classifier_l2_KL}
    h^*(x) = \frac{\omega y_0(x) \dif P_0(x) + (1-\omega) y_1(x) \dif P_1(x)}{\omega \dif P_0(x) + (1-\omega) \dif P_1(x)},
\end{align}
which leads to the equivalent expression of $\epsilon_{\text{split},2}$:
\begin{align*}
    \sup_{\omega\in[0,1]} \omega(1-\omega) \int \frac{(y_1(x)-y_0(x))^2 \dif P_0(x)\dif P_1(x)}{\omega \dif P_0(x) + (1-\omega) \dif P_1(x)}.
\end{align*}
When the KL-loss is used, the optimal classifier $h^*$ has expression in \eqref{eq::opt_classifier_l2_KL} as well. Consequently, we have the equivalent expression of $\epsilon_{\text{split},\mathsf{KL}}$:
\begin{align}
\label{eq::KL_orig_sup_omega}
    \sup_{\omega\in [0,1]} \omega \EE{\KL(y_0(X)\|h^*(X)) \mid S=0} + (1-\omega) \EE{\KL(y_1(X)\|h^*(X)) \mid S=1}.
\end{align}
This expression can be further simplified by using the chain rule of KL-divergence:
\begin{align*}
    \KL(Q_{X,Y} \| R_{X,Y})
    = \KL(Q_{Y|X} \| R_{Y|X} \mid Q_X) + \KL(Q_X \| R_X).
\end{align*}
By taking $\dif Q_{X} = \dif P_{s}$, $\dif R_X = w\dif P_0 + (1-w)\dif P_1$, $Q_{Y|X}(1|x) = y_s(x)$, and $R_{Y|X}(1|x) = h^*(x)$, we obtain
\begin{align*}
    &\EE{\KL(y_s(X)\|h^*(X)) \mid S=s}\\
    &= \KL\left(P_{X,Y|S=s} \| \omega P_{X,Y|S=0} + (1-\omega) P_{X,Y|S=1} \right) - \KL\left(P_s \| \omega P_0 + (1-\omega)P_1 \right).
\end{align*}
Substituting this into \eqref{eq::KL_orig_sup_omega} gives
\begin{align*}
    \epsilon_{\text{split},\mathsf{KL}}
    = \sup_{\omega\in[0,1]}\mathsf{JS}_{\omega}(P_{X,Y|S=0} \| P_{X,Y|S=1}) - \mathsf{JS}_{\omega}(P_{0} \| P_{1}),
\end{align*}
where $\mathsf{JS}_{\omega}(\cdot \|\cdot)$ is the Jensen-Shannon divergence (see \eqref{eq::defn_JS_general} for its definition).
\end{proof}

\subsection{Proof of Theorem~\ref{thm::ben_split_l12_loss_bounds}}
\label{sec::proof_bounds_eps_12}

We divide the proof of Theorem~\ref{thm::ben_split_l12_loss_bounds} into three independent steps. First, we prove the upper bounds for $\epsilon_{\text{split},1}$, $\epsilon_{\text{split},2}$, and $\epsilon_{\text{split},\mathsf{KL}}$ in a unified way. Then we prove the lower bounds for $\epsilon_{\text{split},1}$ and $\epsilon_{\text{split},\mathsf{KL}}$ using Lemma~\ref{lem::ben_split_l12_loss_equivalent}. Finally, we prove the lower bound for $\epsilon_{\text{split},2}$ by leveraging the proof techniques of Brown-Low's two-points lower bound \citep{brown1996constrained}.
\begin{proof}
Note that \eqref{eq::max_inf_loss_0} implies in the information-theoretic regime, optimal split classifiers can always achieve perfect performance. Specifically, one can select labeling functions $y_0$ and $y_1$ as split classifiers which have zero loss on each group. Hence, the problem remains upper bounding the performance of the optimal group-blind classifier. To achieve this goal, we consider two special group-blind classifiers:
\begin{align}
    &h^*(x) 
    = \frac{\dif P_0(x)}{2\dif P(x)} y_0(x) + \frac{\dif P_1(x)}{2\dif P(x)} y_1(x), \label{eq::h*_dP0dP}\\
    &h^{**}(x)
    = \frac{1}{2}(y_0(x) + y_1(x)), \label{eq::h**_average}
\end{align}
where $\dif P = (\dif P_0 + \dif P_1)/2$. In what follows, we upper bound the performance of the group-blind classifiers in \eqref{eq::h*_dP0dP} and \eqref{eq::h**_average} and these bounds will be naturally translated into the upper bounds of $\epsilon_{\text{split},1}$, $\epsilon_{\text{split},2}$, and $\epsilon_{\text{split},\mathsf{KL}}$, respectively.

We upper bound $\epsilon_{\text{split},1}$ by using the group-blind classifier $h^*$ in \eqref{eq::h*_dP0dP}.
\begin{align}
    \epsilon_{\text{split},1}
    =\inf_{h:\mathcal{X}\to [0,1]} \max_{s\in\{0,1\}} \EE{\lvert h(X)-y_s(X)\rvert \mid S=s}
    &\leq \max_{s\in\{0,1\}} \EE{\lvert h^*(X)-y_s(X)\rvert \mid S=s}\nonumber\\
    &= \int |y_1(x)-y_0(x)| \tfrac{\dif P_1(x)}{2 \dif P(x)} \dif P_0(x). \label{eq::inf_max_1_ub_1}
\end{align}
By the Cauchy-Schwarz inequality, we can further upper bound \eqref{eq::inf_max_1_ub_1} by 
\begin{align}
    \sqrt{\EE{(y_1(X)-y_0(X))^2 \mid S=0} \cdot \int \left(\tfrac{\dif P_1(x)}{2\dif P(x)}\right)^2 \dif P_0(x)}. \label{eq::inf_max_1_ub_2}
\end{align}
Furthermore, we have 
\begin{align}
    \int \left(\tfrac{\dif P_1(x)}{2\dif P(x)}\right)^2 \dif P_0(x)
    = \frac{1}{4}\int \left(\tfrac{\dif P_1(x)}{\dif P(x)}\right)^2 \tfrac{\dif P_0(x)}{\dif P(x)} \dif P(x)
    = \frac{1}{4} \int \left(\tfrac{\dif P_1(x)}{\dif P(x)}\right)^2 \left(2-\tfrac{\dif P_1(x)}{\dif P(x)}\right) \dif P(x). \label{eq::inf_max_tv_ub1}
\end{align}
Since $\frac{1}{4} x^2(2-x) \leq 1-|x-1|$ holds for any $x\geq 0$, the RHS of \eqref{eq::inf_max_tv_ub1} can be upper bounded by
\begin{align}
    1 - \int \left|\tfrac{\dif P_1(x)}{\dif P(x)} - 1\right| \dif P(x) = 1- \frac{1}{2} \int |\dif P_1(x) - \dif P_0(x)| = 1-\TV(P_0\|P_1). \label{eq::inf_max_tv_ub2}
\end{align}
Combining (\ref{eq::inf_max_1_ub_1}--\ref{eq::inf_max_tv_ub2}) gives 
\begin{align*}
    \epsilon_{\text{split},1}
    \leq \sqrt{\EE{(y_1(X)-y_0(X))^2 \mid S=0}}\cdot \sqrt{1-\TV(P_0\|P_1)}.
\end{align*}
By symmetry, we can further tighten the upper bound  
\begin{align*}
    \epsilon_{\text{split},1}
    \leq \min_{s\in\{0,1\}} \sqrt{\EE{(y_1(X)-y_0(X))^2\mid S=s}} \cdot \sqrt{1-\TV(P_0\| P_1)}.
\end{align*}
On the other hand, using the classifier $h^{**}$ in \eqref{eq::h**_average} leads to an alternative upper bound 
\begin{align*}
    \epsilon_{\text{split},1} 
    \leq \frac{1}{2} \max_{s\in \{0,1\}} \EE{\lvert y_1(X)-y_0(X)\rvert \mid S=s}.
\end{align*}

Similarly, we can upper bound $\epsilon_{\text{split},2}$ by using the classifier $h^*$ in \eqref{eq::h*_dP0dP}
\begin{align}
    \epsilon_{\text{split},2}
    &\leq \max_{s\in\{0,1\}} \EE{(h^*(X)-y_s(X))^2 \mid S=s}\nonumber\\
    &\leq \int (y_1(x)-y_0(x))^2 \left(\tfrac{\dif P_1(x)}{2 \dif P(x)}\right)^2 \dif P_0(x) + \int (y_1(x)-y_0(x))^2 \left(\tfrac{\dif P_0(x)}{2 \dif P(x)}\right)^2 \dif P_1(x) \nonumber\\
    &= \int (y_1(x)-y_0(x))^2 \tfrac{\dif P_1(x)}{2 \dif P(x)} \dif P_0(x), \label{eq::ub_eps_2_y01_dP1dP}
\end{align}
where the second inequality uses the fact that $\max\{a,b\}\leq a+b$. By the Cauchy-Schwarz inequality and \eqref{eq::inf_max_tv_ub1}, \eqref{eq::inf_max_tv_ub2}, we can further upper bound \eqref{eq::ub_eps_2_y01_dP1dP} by 
\begin{align*}
    \sqrt{\EE{(y_1(X)-y_0(X))^4 \mid S=0}}\cdot \sqrt{1-\TV(P_0\|P_1)}.
\end{align*}
By symmetry, we can further tighten this upper bound by replacing it with
\begin{align*}
    \min_{s\in\{0,1\}}\sqrt{\EE{(y_1(X)-y_0(X))^4 \mid S=s}}\cdot \sqrt{1-\TV(P_0\|P_1)}.
\end{align*}
On the other hand, using the classifier $h^{**}$ in \eqref{eq::h**_average} leads to an alternative upper bound for $\epsilon_{\text{split},2}$. 

We repeat the same strategy and upper bound $\epsilon_{\text{split}, \mathsf{KL}}$ by using the classifier $h^*$ in \eqref{eq::h*_dP0dP}
\begin{align*}
    \epsilon_{\text{split},\mathsf{KL}}
    &\leq  \max_{s\in\{0,1\}} \EE{\KL(y_s(X)\|h^*(X))\mid S=s}\\
    &\leq \EE{\KL(y_0(X)\|h^*(X))\mid S=0} + \EE{\KL(y_1(X)\|h^*(X))\mid S=1}.
\end{align*}
Recall the chain rule of KL-divergence
\begin{align*}
    \KL(Q_{X,Y} \| R_{X,Y})
    = \KL(Q_{Y|X} \| R_{Y|X} \mid Q_X) + \KL(Q_X \| R_X).
\end{align*}
By taking $\dif Q_{X} = \dif P_{s}$, $\dif R_X = \dif P$, $Q_{Y|X}(1|x) = y_s(x)$, and $R_{Y|X}(1|x) = h^*(x)$ and noticing the definition of $h^*$ in \eqref{eq::h*_dP0dP}, we obtain
\begin{align}
\label{eq::lb_KL_eq2}
    \EE{\KL(y_s(X)\|h^*(X)) \mid S=s}
    = \KL\left(P_{X,Y|S=s} \| \tfrac{P_{X,Y|S=0} + P_{X,Y|S=1}}{2} \right) - \KL\left(P_s \| \tfrac{P_0 + P_1}{2} \right).
\end{align}
Hence, 
\begin{align*}
    \epsilon_{\text{split},\mathsf{KL}}
    \leq 2\mathsf{JS}(P_{X,Y|S=0} \| P_{X,Y|S=1}) - 2\mathsf{JS}(P_{0} \| P_{1}),
\end{align*}
where $\mathsf{JS}(\cdot \|\cdot )$ is the Jensen-Shannon divergence. On the other hand, taking the classifier $h^{**}$ in \eqref{eq::h**_average} gives an alternative upper bound for $\epsilon_{\text{split},\mathsf{KL}}$. 
\end{proof}

We proceed to prove the lower bounds of $\epsilon_{\text{split},1}$ and $\epsilon_{\text{split},\mathsf{KL}}$.

\begin{proof}
Recall that $\mathcal{A}_{0.5} \defined \left\{x\in\mathcal{X} \mid \frac{\dif P_0(x)}{\dif P_1(x)} \geq 1 \right\}$. By Lemma~\ref{lem::ben_split_l12_loss_equivalent}, we have 
\begin{align*}
    &\epsilon_{\text{split},1}
    \geq \frac{1}{2} \left(\int_{\mathcal{A}_{0.5}} |y_1(x) - y_0(x)| \dif P_1(x) + \int_{\mathcal{A}_{0.5}^c} |y_1(x) - y_0(x)| \dif P_0(x) \right)\\
    &= \frac{1}{2} \left(\EE{\lvert y_1(X)-y_0(X)\rvert \mid S=1} - \int_{\mathcal{A}_{0.5}^c} |y_1(x) - y_0(x)| (\dif P_1(x)-\dif P_0(x)) \right)\\
    &= \frac{1}{2} \left(\EE{\lvert y_1(X)-y_0(X)\rvert \mid S=1} - \int |y_1(x) - y_0(x)| \left(1 - \tfrac{\dif P_0(x)}{\dif P_1(x)}\right)_{+} \dif P_1(x) \right)\\
    &\geq \frac{1}{2} \left(\EE{\lvert y_1(X)-y_0(X)\rvert \mid S=1} - \sqrt{\int (y_1(x) - y_0(x))^2 \dif P_1(x) \int \left(1 - \tfrac{\dif P_0(x)}{\dif P_1(x)}\right)^2 \indicator{\left[\tfrac{\dif P_0(x)}{\dif P_1(x)}\leq 1\right]} \dif P_1(x)} \right)\\
    &= \frac{1}{2} \left(\EE{\lvert y_1(X)-y_0(X)\rvert \mid S=1} - \sqrt{\EE{(y_1(X)-y_0(X))^2 \mid S=1}} \cdot d_2(P_0\|P_1) \right),
\end{align*}
where $d_2(P_0\|P_1)$ is Marton's divergence. By symmetry, one can obtain 
\begin{equation}
\label{eq::ben_ell1_lb_infmax}
\begin{aligned}
    \epsilon_{\text{split},1}
    &\geq \frac{1}{2} \max_{s\in \{0,1\}} \left(\EE{\lvert y_1(X)-y_0(X)\rvert \mid S=s} - \sqrt{\EE{(y_1(X)-y_0(X))^2 \mid S=s}} \cdot d_2(P_{1-s}\|P_s)\right).
\end{aligned}
\end{equation}
Finally, the lower bound of $\epsilon_{\text{split},\mathsf{KL}}$ follows directly from Lemma~\ref{lem::ben_split_l12_loss_equivalent}.
\end{proof}

Before getting to the lower bound of $\epsilon_{\text{split},2}$, we prove a useful lemma. Here we denote 
\begin{subequations}
\label{eq::def_A_and_B_L2_BOS}
\begin{align}
    A_s &\defined \EE{\lvert y_1(X) - y_0(X)\rvert \mid S=s} \quad \text{for } s\in \{0,1\},\\
    B &\defined \sqrt{\chidiv(P_{1} \| P_{0}) + 1}.
\end{align}
\end{subequations}
\begin{lem}
\label{lem::l2loss_impos_da}
Assume that $A_0 \leq A_1$. For any measurable classifier $h: \mathcal{X}\to [0,1]$ and constant $0 \leq \epsilon < A_0^2/B^2$, if $\EE{(h(X)-y_0(X))^2 \mid S=0}\leq \epsilon$, then $\EE{(h(X)-y_1(X))^2 \mid S=1} \geq (A_1 - B \sqrt{\epsilon})^2$.
\end{lem}
\begin{proof}
Consider a convex optimization problem
\begin{align*}
    \min_{h:\mathcal{X}\to [0,1]}~& \int (h(x)-y_1(x))^2 \dif P_{1}(x),\\
    \sto~& \int (h(x)-y_0(x))^2 \dif P_{0}(x) \leq \epsilon.
\end{align*}
Computing the Gateaux derivative of the Lagrange multiplier gives the following optimal conditions \citep[][Theorem~6.6.1]{kurdila2006convex},
\begin{align}
    (h(x)-y_1(x)) \dif P_{1}(x) + \lambda (h(x)-y_0(x)) \dif P_{0}(x) &= 0,\\
    \lambda \left(\int (h(x)-y_0(x))^2 \dif P_{0}(x) - \epsilon \right) &= 0, \label{eq::comp_slack_soft_class}\\
    \lambda &\geq 0,
\end{align}
which provides the optimal classifier 
\begin{align*}
    h^*(x) = \frac{y_1(x) \dif P_{1}(x) + \lambda y_0(x) \dif P_{0}(x)}{\dif P_{1}(x) + \lambda \dif P_{0}(x)}.
\end{align*}
We denote $r(x) \defined \frac{\dif P_{1}(x)}{\dif P_{0}(x)}$ and simplify the expression of the optimal classifier
\begin{align}
    h^*(x) = \frac{y_1(x)r(x) + \lambda y_0(x)}{r(x) + \lambda}.\label{eq::opt_class_soft_class}
\end{align}
If $\lambda = 0$, then $h^*(x)=y_1(x)$ and, consequently,
\begin{align*}
    \EE{(h^*(X)-y_0(X))^2 \mid S=0} = \EE{(y_1(X)-y_0(X))^2 \mid S=0}.
\end{align*}
However, this contradicts our assumptions $\EE{(h^*(X)-y_0(X))^2\mid S=0} \leq \epsilon$ and 
\begin{align*}
    \epsilon 
    < \frac{\EE{\lvert y_1(X) - y_0(X)\rvert \mid S=0}^2}{\chidiv(P_{1} \| P_{0}) + 1} 
    \leq \EE{(y_1(X) - y_0(X))^2 \mid S=0}.
\end{align*}
Hence, we have $\lambda > 0$. In this case,  \eqref{eq::comp_slack_soft_class} and \eqref{eq::opt_class_soft_class} imply
\begin{align*}
    \int \left(\frac{y_1(x)r(x) + \lambda y_0(x)}{r(x) + \lambda} - y_0(x)\right)^2 \dif P_0(x) = \epsilon.
\end{align*}
We simplify the expression and obtain 
\begin{align}
\label{eq::proof_eps_r_ydif}
    \int r(x)^2 \left(\frac{y_1(x) - y_0(x)}{r(x) + \lambda}\right)^2 \dif P_0(x) = \epsilon.
\end{align}
Now we consider lower bounding $\EE{(h^*(X)-y_1(X))^2\mid S=1}$. By its definition and the expression of the optimal classifier~\eqref{eq::opt_class_soft_class}, we have 
\begin{align}
    \EE{(h^*(X)-y_1(X))^2\mid S=1} 
    &= \int \left(\frac{y_1(x)r(x) + \lambda y_0(x)}{r(x) + \lambda}-y_1(x)\right)^2 \dif P_{1}(x)\nonumber\\
    &= \int \left(\frac{\lambda (y_1(x) - y_0(x))}{r(x) + \lambda}\right)^2 \dif P_{1}(x) \nonumber\\
    &\geq \left(\int \frac{\lambda |y_1(x) - y_0(x)|}{r(x) + \lambda} \dif P_{1}(x) \right)^2\nonumber\\
    &= \left(\int |y_1(x) - y_0(x)| \dif P_{1}(x) - \int \frac{r(x) |y_1(x) - y_0(x)|}{r(x) + \lambda} \dif P_{1}(x) \right)^2\nonumber\\
    &= \left(\EE{\lvert y_1(X) - y_0(X)\rvert \mid S=1} - \int \frac{r(x) |y_1(x) - y_0(x)|}{r(x) + \lambda} \dif P_{1}(x) \right)^2,  \label{eq::proof_L1_low_bound_1}
\end{align}
where the only inequality is due to the Cauchy-Schwarz inequality. Furthermore, by the Cauchy-Schwarz inequality again and \eqref{eq::proof_eps_r_ydif}, we have 
\begin{align}
   \int \frac{r(x) |y_1(x) - y_0(x)|}{r(x) + \lambda} \dif P_{1}(x)
   &\leq \sqrt{\int r(x)^2\left(\frac{y_1(x) - y_0(x)}{r(x) + \lambda}\right)^2 \dif P_{0}(x) \int r(x) \dif P_{1}(x) } \nonumber \\
   &= \sqrt{\epsilon \EE{r(X) \mid S=1}}. \label{eq::proof_eps_CS_ineq_1}
\end{align}
Recall that $r(x) = \frac{\dif P_{1}(x)}{\dif P_{0}(x)}$. Hence,
\begin{align}
    \EE{r(X) \mid S=1} 
    = \int \frac{\dif P_{1}(x)}{\dif P_{0}(x)} \dif P_{1}(x)
    = \int \left[\left(\frac{\dif P_{1}(x)}{\dif P_{0}(x)}\right)^2-1\right] \dif P_{0}(x) + 1
    = \chidiv(P_{1} \| P_{0}) + 1. \label{eq::ratio_chi_div}
\end{align}
By our assumptions, 
\begin{align*}
    &\EE{\lvert y_1(X) - y_0(X)\rvert \mid S=1} - \sqrt{\epsilon (\chidiv(P_{1} \| P_{0}) + 1)}\\
    &\geq \EE{\lvert y_1(X) - y_0(X)\rvert \mid S=1} - \EE{\lvert y_1(X) - y_0(X)\rvert \mid S=0}
    \geq 0.
\end{align*}
Combining \eqref{eq::proof_L1_low_bound_1}, \eqref{eq::proof_eps_CS_ineq_1}, and \eqref{eq::ratio_chi_div} together, we conclude that 
\begin{align*}
    \EE{(h^*(X)-y_1(X))^2\mid S=1}
    &\geq (\EE{\lvert y_1(X) - y_0(X)\rvert \mid S=1} - \sqrt{\epsilon \EE{r(X) \mid S=1}})^2 \\
    &= \left(\EE{\lvert y_1(X) - y_0(X)\rvert \mid S=1} - \sqrt{\epsilon \left(\chidiv(P_{1} \| P_{0} \right) + 1)} \right)^2.
\end{align*}
\end{proof}

Now we are in a position to prove the lower bound for $\epsilon_{\text{split},2}$.
\begin{proof}
By Lemma~\ref{lem::l2loss_impos_da}, for any classifier $h: \mathcal{X}\to [0,1]$ and 
\begin{align*}
    0 \leq \epsilon < \frac{A_0^2}{B^2},
\end{align*}
if $\EE{(h(X)-y_0(X))^2 \mid S=0}\leq \epsilon$, then $\EE{(h(X)-y_1(X))^2 \mid S=1} \geq (A_1 - B \sqrt{\epsilon})^2$, where $A_0$, $A_1$, and $B$ are defined in \eqref{eq::def_A_and_B_L2_BOS}.
Now we take $\epsilon = (A_0/(B+1))^2$, which naturally satisfies the assumptions in Lemma~\ref{lem::l2loss_impos_da}. As a result, if
\begin{align*}
    \EE{(h(X)-y_0(X))^2\mid S=0} \leq \left(\frac{A_0}{B+1}\right)^2,
\end{align*}
then 
\begin{align*}
    \EE{(h(X)-y_1(X))^2 \mid S=1}
    \geq \left(A_1 - B \frac{A_0}{B+1}\right)^2
    \geq \left(\frac{A_0}{B+1}\right)^2
\end{align*}
where the second inequality is because of the assumption $A_1 \geq A_0$. Consequently, for any $h:\mathcal{X}\to [0,1]$, 
\begin{align*}
    \max_{s\in \{0,1\}} \EE{(h(X)-y_s(X))^2\mid S=s}
    \geq \left(\frac{A_0}{B+1}\right)^2
    = \left(\frac{\EE{\lvert y_1(X) - y_0(X)\rvert \mid S=0}}{\sqrt{\chidiv(P_{1} \| P_{0}) + 1} + 1}\right)^2.
\end{align*}
\end{proof}

\subsection{Extension to Cross Entropy Loss}
\label{appendix::cross_ent}

As discussed in Remark~\ref{rem::cross_ent_split}, one can define the benefit-of-splitting under the cross entropy loss. We provide upper and lower bounds of the corresponding benefit-of-splitting as an extension of Theorem~\ref{thm::ben_split_l12_loss_bounds}. 
\begin{prop}
The cross-entropy-benefit-of-splitting can be upper and lower bounded
\begin{align*}
    &\epsilon_{\text{split},\mathsf{H}}
    \leq \min\left\{2\mathsf{JS}(P_{X,Y|S=0} \| P_{X,Y|S=1}) - 2\mathsf{JS}(P_{0} \| P_{1}),\  \max_{s\in \{0,1\}} \EE{\KL\left(y_s(X) \| \tfrac{y_0(X) + y_1(X)}{2}\right) \mid S=s} \right\},\\
    &\epsilon_{\text{split},\mathsf{H}}
    \geq \mathsf{JS}(P_{X,Y|S=0} \| P_{X,Y|S=1}) - \mathsf{JS}(P_{0} \| P_{1}) - \frac{1}{2} \left|\EE{\mathsf{H}(y_0(X)) \mid S=0} - \EE{\mathsf{H}(y_1(X)) \mid S=1}\right|.
\end{align*}
\end{prop}

\begin{proof}
Recall that $\mathsf{H}(p,q) = \KL(p\|q) + \mathsf{H}(p)$. Hence,
\begin{align*}
    \EE{\mathsf{H}(y_s(X), h(X)) \mid S=s}
    = \EE{\KL(y_s(X)\| h(X)) \mid S=s} + \EE{H(y_s(X))\mid S=s},
\end{align*}
which leads to 
\begin{align}
\label{eq::cross_ent_proof_eq1}
    \max_{s\in \{0,1\}}\inf_{h:\mathcal{X}\to [0,1]} \EE{\mathsf{H}(y_s(X), h(X)) \mid S=s}
    = \max_{s\in \{0,1\}}\EE{H(y_s(X))\mid S=s}.
\end{align}
Since $\max\{a_i+b_i\}\leq \max\{a_i\} + \max\{b_i\}$,
\begin{align}
    &\inf_{h: \mathcal{X}\to [0,1]} \max_{s\in \{0,1\}} \EE{\mathsf{H}(y_s(X), h(X)) \mid S=s} \nonumber\\
    &\leq \inf_{h: \mathcal{X}\to [0,1]} \max_{s\in \{0,1\}} \EE{\mathsf{KL}(y_s(X)\| h(X)) \mid S=s} + \max_{s\in \{0,1\}} \EE{\mathsf{H}(y_s(X)) \mid S=s}. \label{eq::cross_ent_proof_eq2}
\end{align}
Combining \eqref{eq::cross_ent_proof_eq1} and \eqref{eq::cross_ent_proof_eq2} implies $\epsilon_{\text{split},\mathsf{H}} \leq \epsilon_{\text{split},\mathsf{KL}}$. Therefore, the upper bound of $\epsilon_{\text{split},\mathsf{KL}}$ in Theorem~\ref{thm::ben_split_l12_loss_bounds} is an upper bound of $\epsilon_{\text{split},\mathsf{H}}$ as well. Now we proceed to prove the lower bound.
\begin{align}
    &\inf_{h: \mathcal{X}\to [0,1]} \max_{s\in \{0,1\}} \EE{\mathsf{H}(y_s(X), h(X)) \mid S=s}\nonumber \\
    &\geq \frac{1}{2} \inf_{h: \mathcal{X}\to [0,1]} \left(\EE{\mathsf{H}(y_0(X), h(X)) \mid S=0} + \EE{\mathsf{H}(y_1(X), h(X)) \mid S=1} \right) \nonumber\\
    \begin{split}
    \label{eq::cross_ent_proof_eq3}
    &= \frac{1}{2} \inf_{h: \mathcal{X}\to [0,1]} \left(\EE{\KL(y_0(X)\| h(X)) \mid S=0} + \EE{\KL(y_1(X)\| h(X)) \mid S=1} \right) \\
    &\quad + \frac{1}{2} \left(\EE{\mathsf{H}(y_0(X)) \mid S=0} + \EE{\mathsf{H}(y_1(X)) \mid S=1} \right).
    \end{split}
\end{align}
By the proof of Lemma~\ref{lem::ben_split_l12_loss_equivalent}, we have 
\begin{align}
    &\frac{1}{2} \inf_{h: \mathcal{X}\to [0,1]} \left(\EE{\KL(y_0(X)\| h(X)) \mid S=0} + \EE{\KL(y_1(X)\| h(X)) \mid S=1} \right)\nonumber\\
    &= \mathsf{JS}(P_{X,Y|S=0} \| P_{X,Y|S=1}) - \mathsf{JS}(P_{0} \| P_{1}). \label{eq::cross_ent_proof_eq4}
\end{align}
Combining \eqref{eq::cross_ent_proof_eq1}, \eqref{eq::cross_ent_proof_eq3}, \eqref{eq::cross_ent_proof_eq4} gives
\begin{align*}
    \epsilon_{\text{split},\mathsf{H}} 
    &\geq \mathsf{JS}(P_{X,Y|S=0} \| P_{X,Y|S=1}) - \mathsf{JS}(P_{0} \| P_{1}) \\
    &\quad + \frac{1}{2}\left(\EE{\mathsf{H}(y_0(X)) \mid S=0} + \EE{\mathsf{H}(y_1(X)) \mid S=1} \right) - \max_{s\in\{0,1\}} \EE{\mathsf{H}(y_s(X)) \mid S=s}\\
    &= \mathsf{JS}(P_{X,Y|S=0} \| P_{X,Y|S=1}) - \mathsf{JS}(P_{0} \| P_{1}) - \frac{1}{2} \left|\EE{\mathsf{H}(y_0(X)) \mid S=0} - \EE{\mathsf{H}(y_1(X)) \mid S=1}\right|.
\end{align*}
\end{proof}

\subsection{Proof of Proposition~\ref{prop::bounds_popu_ben_split}}
\label{append::pop_ben_split}

\begin{proof}
First, note that $\inf_{\substack{h_s:\mathcal{X}\to [0,1]\\ \text{for } s\in\{0,1\}}} \EE{|h_S(X)-y_S(X)|} = 0$ as one can choose $h_s(x) = y_s(x)$. Hence, our focus is on upper and lower bounding
\begin{align}
    \epsilon_{\text{split},\text{pop}}
    &= \inf_{h:\mathcal{X}\to [0,1]} \EE{|h(X)-y_S(X)|} \nonumber\\
    &= \inf_{h:\mathcal{X}\to [0,1]} \Pr(S=0) \int |h(x)-y_0(x)| \dif P_0(x) + \Pr(S=1) \int |h(x)-y_1(x)| \dif P_1(x). \label{eq::inf_popu_eq_1}
\end{align}
By the proof of Lemma~\ref{lem::ben_split_l12_loss_equivalent}, the optimal classifier of \eqref{eq::inf_popu_eq_1} is 
\begin{align*}
    h^*(x) = 
    \begin{cases}
        y_0(x) &\text{ if } \frac{\Pr(S=0)\cdot \dif P_0(x)}{\Pr(S=1)\cdot \dif P_1(x)} \geq 1\\
        y_1(x) &\text{ otherwise}.
    \end{cases}
\end{align*}
By plugging the optimal classifier into \eqref{eq::inf_popu_eq_1}, we can write $\epsilon_{\text{split},\text{pop}}$ equivalently as
\begin{align*}
    \Pr(S=0)\cdot \EE{\lvert y_1(X) - y_0(X)\rvert \mid S=0} - \int |y_1(x)-y_0(x)| \left(\Pr(S=0)\cdot \dif P_0(x) - \Pr(S=1)\cdot \dif P_1(x)\right)_{+}.
\end{align*}
The desired upper bound can be obtained by dropping the negative term. Now we proceed to prove the lower bound. Since
\begin{align*}
    &\int |y_1(x)-y_0(x)| \left(\Pr(S=0)\cdot \dif P_0(x) - \Pr(S=1)\cdot \dif P_1(x)\right)_{+}\\
    &\leq \Pr(S=0) \int \left(\tfrac{\dif P_0(x)}{\dif P_1(x)} - \tfrac{\Pr(S=1)}{\Pr(S=0)} \right)_{+} \dif P_1(x)\\
    &= \Pr(S=0)\cdot E_{\frac{\Pr(S=1)}{\Pr(S=0)}}(P_0\|P_1),
\end{align*}
where $E_{\frac{\Pr(S=1)}{\Pr(S=0)}}(P_0\|P_1)$ is the $E_{\gamma}$-divergence with $\gamma = \Pr(S=1)/\Pr(S=0)$, we have
\begin{align*}
    \epsilon_{\text{split},\text{pop}}
    \geq \Pr(S=0) \left(\EE{\lvert y_1(X) - y_0(X)\rvert \mid S=0} - E_{\frac{\Pr(S=1)}{\Pr(S=0)}}(P_0\|P_1)\right).
\end{align*}
\end{proof}

\section{Proofs for Section~\ref{sec::perf_imp_split}}

\subsection{Proof of Theorem~\ref{thm::FER_split_dual}}
\label{append::FER_split}

Recall that the false error rate is the maximum between false positive rate and false negative rate
\begin{align}
    \FER{s}{h} \defined \max\left\{\EE{h(X)\mid Y=0, S=s},\ \EE{1-h(X)\mid Y=1, S=s}\right\}.
\end{align}
We prove the following lemma which will be used in the proof of Theorem~\ref{thm::FER_split_dual}.
\begin{lem}
\label{lem::FER_equivalent_exp}
The false error rate has the following equivalent expressions
\begin{align*}
    \FER{s}{h}
    &= \max\left\{\frac{\EE{h(X)(1-y_s(X))\mid S=s}}{\Pr(Y=0 \mid S=s)}, 1- \frac{\EE{h(X)y_s(X)\mid S=s}}{\Pr(Y=1\mid S=s)} \right\}\\
    &= \max\left\{\frac{\EE{h(X)(1-y_s(X))f_s(X)}}{\Pr(Y=0 \mid S=s)}, 1- \frac{\EE{h(X)y_s(X)f_s(X)}}{\Pr(Y=1\mid S=s)} \right\}.
\end{align*}
where $f_s(x) \defined \frac{\Pr(S=s \mid X=x)}{\Pr(S=s)}$.
\end{lem}
\begin{proof}
The proof follows directly from Bayes's rule,
\begin{align*}
    \dif P_{X|Y=0,S=s} 
    = \frac{1-y_s(x)}{\Pr(Y=0\mid S=s)} \dif P_{X|S=s}
    = \frac{1-y_s(x)}{\Pr(Y=0\mid S=s)}\cdot f_s(x) \dif P_{X}.
\end{align*}
\end{proof}

Now we are in a position to prove Theorem~\ref{thm::FER_split_dual}.

\begin{proof}
By Lemma~\ref{lem::FER_equivalent_exp}, the quantity $\inf_{h: \mathcal{X}\to [0,1]}\max_{s\in\{0,1\}} \FER{s}{h}$ can be equivalently written as 
\begin{align}
    \inf_{h: \mathcal{X}\to [0,1]}\max_{s\in\{0,1\}} \left\{\frac{\EE{h(X)(1-y_s(X))f_s(X)}}{\Pr(Y=0\mid S=s)}, 1- \frac{\EE{h(X)y_s(X)f_s(X)}}{\Pr(Y=1\mid S=s)} \right\} 
    =\inf_{h:\mathcal{X}\to [0,1]}
    \max_{\bmu \in \Delta_4}
    G(\bmu,h), \label{eq::FER_infmax_infsup_prior}
\end{align}
where $\bmu \defined (\mu_{0,0}, \mu_{0,1},\mu_{1,0},\mu_{1,1})$ and 
\begin{align*}
    G(\bmu,h)
    &\defined \sum_{s\in \{0,1\}} \left(\mu_{s,0}\frac{\EE{h(X)(1-y_s(X))f_s(X)}}{\Pr(Y=0 \mid S=s)} + \mu_{s,1}\left(1- \frac{\EE{h(X)y_s(X)f_s(X)}}{\Pr(Y=1\mid S=s)}\right) \right)\\
    &= \sum_{s\in \{0,1\}}\mu_{s,1} + \EE{\sum_{s\in \{0,1\}} \left(\frac{\mu_{s,0}(1-y_s(X))f_s(X)}{\Pr(Y=0 \mid S=s)} - \frac{\mu_{s,1}y_s(X)f_s(X)}{\Pr(Y=1\mid S=s)} \right)h(X)}.
\end{align*}
By denoting 
\begin{align*}
    \phi_{s,0}(x) 
    \defined \frac{(1-y_s(x))f_s(x)}{\Pr(Y=0 \mid S=s)}\quad
    \phi_{s,1}(x)
    \defined \frac{-y_s(x)f_s(x)}{\Pr(Y=1\mid S=s)},
\end{align*}
we can write
\begin{align*}
    G(\bmu,h)
    = \sum_{s\in \{0,1\}}\mu_{s,1} + \EE{\sum_{s,i\in \{0,1\}} \mu_{s,i}\phi_{s,i}(X) h(X)}.
\end{align*}
We next use Ky Fan's min-max theorem~\citep{fan1953minimax} (see Lemma~\ref{lem::KyFan}) to swap the positions of infimum and maximum. First, $\Delta_4$ is a compact set and for any $h:\mathcal{X}\to [0,1]$, $G(\cdot,h)$ is continuous on $\Delta_4$. Furthermore, for any $h:\mathcal{X}\to [0,1]$, $G(\cdot,h)$ is linear over $\Delta_4$; for any $\bmu \in \Delta_4$, $G(\bmu,\cdot)$ is convex-like over all (measurable) classifiers from $\mathcal{X}$ to $[0,1]$. Hence, we have 
\begin{align}
\label{eq::FER_exchange_inf_sup_Lag}
    \inf_{h:\mathcal{X}\to [0,1]}
    \max_{\bmu \in \Delta_4}
    G(\bmu,h) 
    = \max_{\bmu \in \Delta_4} \inf_{h:\mathcal{X}\to [0,1]}
    G(\bmu,h).
\end{align}
Next, we prove that, for any fixed $\bmu\in \Delta_4$,
\begin{align}
\label{eq::inf_Exp_ex_FER}
    \inf_{h:\mathcal{X}\to [0,1]} \EE{\sum_{s,i\in \{0,1\}} \mu_{s,i}\phi_{s,i}(X) h(X)}
    = \EE{\inf_{h:\mathcal{X}\to [0,1]} \sum_{s,i\in \{0,1\}} \mu_{s,i}\phi_{s,i}(X) h(X)}.
\end{align}
One direction $\text{LHS} \geq \text{RHS}$ can be obtained directly since for any $h:\mathcal{X}\to [0,1]$
\begin{align*}
    \EE{\sum_{s,i\in \{0,1\}} \mu_{s,i}\phi_{s,i}(X) h(X)}
    \geq \EE{\inf_{h:\mathcal{X}\to [0,1]} \sum_{s,i\in \{0,1\}} \mu_{s,i}\phi_{s,i}(X) h(X)}.
\end{align*}
Note that the infimum in the RHS of \eqref{eq::inf_Exp_ex_FER} is point-wise. For any fixed $x\in \mathcal{X}$, the following optimization problem
\begin{align*}
    \inf_{\bar{h}\in [0,1]} \sum_{s,i\in \{0,1\}} \mu_{s,i}\phi_{s,i}(x) \bar{h}
\end{align*}
has an optimal solution $\bar{h}^* = \indicator{[\sum_{s,i\in \{0,1\}} \mu_{s,i}\phi_{s,i}(x) \leq 0]}$. Hence, there is a measurable classifier which can achieve the point-wise infimum inside the expectation of the RHS in \eqref{eq::inf_Exp_ex_FER}: $h^*(x) = \indicator{[\sum_{s,i\in \{0,1\}} \mu_{s,i}\phi_{s,i}(x) \leq 0]}$. Consequently, the RHS of \eqref{eq::inf_Exp_ex_FER} can be simplified as 
\begin{align}
\label{eq::FER_RHS_exp_inf}
    \text{RHS} = \EE{\left(\sum_{s,i\in \{0,1\}} \mu_{s,i}\phi_{s,i}(X)\right)_{-}},
\end{align}
where for $a\in \Reals$, $(a)_{-}\defined \min\{a, 0\}$. Since the LHS of \eqref{eq::inf_Exp_ex_FER} is an infimum over all measurable classifiers, using the classifier $h^*$ leads to 
\begin{align*}
    \text{LHS} 
    \leq \EE{\sum_{s,i\in \{0,1\}} \mu_{s,i}\phi_{s,i}(X) h^*(X)}
    = \EE{\left(\sum_{s,i\in \{0,1\}} \mu_{s,i}\phi_{s,i}(X)\right)_{-}}
    =\text{RHS}.
\end{align*}
Combining (\ref{eq::FER_infmax_infsup_prior}--\ref{eq::FER_RHS_exp_inf}) together implies
\begin{align*}
    \inf_{h: \mathcal{X}\to [0,1]}\max_{s\in\{0,1\}} \FER{s}{h} 
    = \max_{\bmu \in \Delta_4} \left\{\sum_{s\in \{0,1\}}\mu_{s,1} + \EE{\left(\sum_{s,i\in \{0,1\}} \mu_{s,i}\phi_{s,i}(X)\right)_{-}} \right\}.
\end{align*}
Similarly, one can prove that
\begin{align*}
    \max_{s\in\{0,1\}} \inf_{h: \mathcal{X}\to [0,1]} \FER{s}{h} 
    = \max_{s\in\{0,1\}} \max_{\bnu^{(s)} \in \Delta_2} \left\{\nu^{(s)}_{1} + \EE{\left(\sum_{i\in \{0,1\}} \nu^{(s)}_{i}\phi_{s,i}(X)\right)_{-}} \right\}.
\end{align*}
\end{proof}

\subsection{Proof of Proposition~\ref{prop::part_deri_FER}}
\label{append::deri_FER}

We start with a useful lemma which will be used in the proof of Proposition~\ref{prop::part_deri_FER}.
\begin{lem}
\label{lem::interchange_supgrad_exp}
Let $f: \mathcal{X}\times \Reals^k \to \Reals$ be a bounded measurable function. For a fixed $x\in\mathcal{X}$, if $v(x, w_0) \in \Reals^k$ is a supergradient of $f(x,\cdot)$ at $w_0$:
\begin{align}
\label{eq::f_supgradient}
    f(x,w) - f(x,w_0) \leq v(x, w_0)^T (w-w_0),
\end{align}
then $\EE{v(X,w_0)}$ is a supergradient of $\EE{f(X,\cdot)}$ at $w_0$:
\begin{align}
    \EE{f(X,w)} - \EE{f(X,w_0)} \leq \EE{v(X, w_0)}^T (w-w_0).
\end{align}
\end{lem}
The proof of Lemma~\ref{lem::interchange_supgrad_exp} follows directly by taking expectation on both sides of \eqref{eq::f_supgradient}. We refer the readers to \citep{rockafellar1982interchange} for a more general result on the interchangeability of subdifferentiation and (conditional) expectation. Now we are in a position to prove Proposition~\ref{prop::part_deri_FER}.

\begin{proof}
Consider a function $g(x,\bmu) \defined \sum_{s\in\{0,1\}}\mu_{s,1} + \left(\sum_{s,i\in \{0,1\}} \mu_{s,i}\phi_{s,i}(x)\right)_{-}$. For a fixed $x$, $g(x,\cdot)$ has a supergradient at $\bmu \defined (\mu_{0,0}, \mu_{0,1},\mu_{1,0},\mu_{1,1})$: 
\begin{align*}
    \left(i+\phi_{s,i}(x)\indicator{\Big[\sum_{s',i'\in\{0,1\}} \mu_{s',i'}\phi_{s',i'}(x) <0 \Big]}\right)_{s,i\in\{0,1\}}.
\end{align*}
Therefore, by Lemma~\ref{lem::interchange_supgrad_exp}, $g$ has a supergradient at $\bmu$: 
\begin{align*}
    &\partial g(\bmu) \ni \left(i + \EE{\phi_{s,i}(X) \cdot \indicator{\Big[\sum_{s',i'\in \{0,1\}} \mu_{s',i'}\phi_{s',i'}(X) < 0\Big]}}\right)_{s,i\in \{0,1\}}.
\end{align*}
Now we introduce auxiliary functions
\begin{align*}
    \psi_{s,i}(x) \defined \frac{1-i-y_s(x)}{\Pr(Y=i\mid S=s)},\quad s,i\in \{0,1\}.
\end{align*}
By Bayes's rule and the definition of $\phi_{s,i}$ (see \eqref{eq::phi_s_i_defn}), we have 
\begin{align*}
    \psi_{s,i}(x)\cdot \dif P_{X|S=s}(x) = \phi_{s,i}(x)\cdot \dif P_{X}(x).
\end{align*}
Hence, 
\begin{align*}
    &\partial g(\bmu) \ni \left(i + \EE{\psi_{s,i}(X) \cdot \indicator{\Big[\sum_{s',i'\in \{0,1\}} \mu_{s',i'}\phi_{s',i'}(X) < 0 \Big]}\ \Bigg|\  S=s}\right)_{s,i\in \{0,1\}}.
\end{align*}
Similarly, one can obtain a closed-form supergradient of $g_s(\bnu)$.
\end{proof}

\section{Proofs for Section~\ref{sec::split_practice}}

\subsection{Proof of Theorem~\ref{thm::H_dep_split_bounds}}
\label{append::H_dep_split_bounds}

We first recall a useful lemma which can be proved by the variational representation \citep{nguyen2010estimating} of total variation distance. 
\begin{lem}
\label{lem::Exp_TV}
For any measurable and non-negative function $f:\mathcal{X}\to \Reals^+$,
\begin{align*}
    |\EE{f(X_0)} - \EE{f(X_1)}| \leq \|f\|_{\infty} \TV(P_{0}\| P_{1}),
\end{align*}
where $X_0\sim P_{0}$ and $X_1\sim P_{1}$.
\end{lem}

Now we are in a position to prove Theorem~\ref{thm::H_dep_split_bounds}. 

\begin{proof}
We prove the upper bound first. Let $h_s^*$ be an optimal classifier for the group $s\in \{0,1\}$, i.e., $h_s^* \in \argmin_{h\in \mathcal{H}} \EE{\lvert h(X)-y_s(X)\rvert \mid S=s}$. Then
\begin{align*}
    \inf_{h\in\mathcal{H}}\max_{s\in \{0,1\}} \EE{\lvert h(X)-y_s(X)\rvert \mid S=s}
    \leq \max\{\EE{\lvert h_0^*(X)-y_1(X)\rvert \mid S=1}, \EE{\lvert h_0^*(X)-y_0(X)\rvert \mid S=0}\}.
\end{align*}
By the triangle inequality, 
\begin{align*}
    \EE{\lvert h_0^*(X)-y_1(X)\rvert \mid S=1}
    &\leq \EE{\lvert h_1^*(X) - h_0^*(X)\rvert \mid S=1} + \EE{\lvert h_1^*(X) - y_1(X)\rvert \mid S=1}.
\end{align*}
Therefore, 
\begin{align*}
    \inf_{h\in\mathcal{H}}\max_{s\in \{0,1\}} \EE{\lvert h(X)-y_s(X)\rvert \mid S=s}
    \leq \EE{\lvert h_1^*(X) - h_0^*(X)\rvert \mid S=1} + \max_{s\in \{0,1\}} \EE{\lvert h_s^*(X)-y_s(X)\rvert \mid S=s},
\end{align*}
which implies that
\begin{align*}
    \epsilon_{\text{split}}^{\mathcal{H}}
    &= \inf_{h\in\mathcal{H}}\max_{s\in \{0,1\}} \EE{\lvert h(X)-y_s(X)\rvert \mid S=s} - \max_{s\in \{0,1\}} \inf_{h\in\mathcal{H}} \EE{\lvert h(X)-y_s(X)\rvert \mid S=s}\\
    &\leq \EE{\lvert h_1^*(X) - h_0^*(X)\rvert \mid S=1}.
\end{align*}
By symmetry, we obtain the desired upper bound for $\epsilon_{\text{split}}^{\mathcal{H}}$. Now we proceed to prove the lower bound for $\epsilon_{\text{split}}^{\mathcal{H}}$. By the triangle inequality, we have 
\begin{align*}
    \EE{\lvert y_1(X) - y_0(X)\rvert \mid S=0} &\geq \EE{\lvert h_1^*(X) - h_0^*(X)\rvert \mid S=0} -\EE{\lvert h_0^*(X) - y_0(X)\rvert \mid S=0} \\
    & \quad 
    - \EE{\lvert h_1^*(X) - y_1(X)\rvert \mid S=0}.
\end{align*}
By Lemma~\ref{lem::Exp_TV},
\begin{align*}
    \EE{\lvert h_1^*(X) - y_1(X) \rvert \mid S=0} 
    \leq \EE{\lvert h_1^*(X) - y_1(X)\rvert \mid S=1} + \TV(P_{0}\| P_{1}).
\end{align*}
Therefore, 
\begin{align*}
    \EE{\lvert y_1(X) - y_0(X)\rvert \mid S=0} &\geq \EE{\lvert h_1^*(X) - h_0^*(X)\rvert \mid S=0} - 2\lambda - \TV(P_{0}\| P_{1}).
\end{align*}
where $\lambda \defined \sum_{s\in \{0,1\}}\EE{\lvert h_s^*(X) - y_s(X)\rvert \mid S=s}/2$. Hence,
\begin{align}
\label{eq::y_10_geq_h10_lamb_TV}
    \max_{s\in \{0,1\}} \EE{\lvert y_1(X) - y_0(X)\rvert \mid S=s}
    &\geq \max_{s\in \{0,1\}} \EE{\lvert h_1^*(X) - h_0^*(X)\rvert \mid S=s} - 2\lambda - \TV(P_{0}\| P_{1}).
\end{align}
By a slight modification of the proof of Theorem~\ref{thm::ben_split_l12_loss_bounds}, we have
\begin{align}
\label{eq::mod_Thm_2_infmax_TV}
    \inf_{h:\mathcal{X}\to [0,1]} \max_{s\in\{0,1\}} \EE{\lvert h(X)-y_s(X)\rvert \mid S=s}
    \geq \frac{1}{2} \left(\max_{s\in \{0,1\}} \EE{\lvert y_1(X)-y_0(X)\rvert \mid S=s} - \TV(P_0\| P_1) \right).
\end{align}
Substituting \eqref{eq::y_10_geq_h10_lamb_TV} into \eqref{eq::mod_Thm_2_infmax_TV} leads to
\begin{align}
    \inf_{h\in\mathcal{H}} \max_{s\in \{0,1\}} \EE{\lvert h(X)-y_s(X)\rvert \mid S=s} 
    &\geq \inf_{h:\mathcal{X}\to [0,1]} \max_{s\in \{0,1\}} \EE{\lvert h(X)-y_s(X)\rvert \mid S=s} \nonumber\\
    &\geq \frac{1}{2} \left(\max_{s\in \{0,1\}} \EE{\lvert h_1^*(X) - h_0^*(X)\rvert \mid S=s} - 2\lambda - 2\TV(P_{0}\| P_{1}) \right). \label{eq::h_ys_h10_lamb_TV}
\end{align}
Finally, since $\max\{a,b\} \leq a+b$ and $\{h_s^*\}_{s\in \{0,1\}}$ is the set of optimal split classifiers, then 
\begin{align}
\label{eq::max_inf_lamb}
    \max_{s\in \{0,1\}} \inf_{h\in\mathcal{H}}\EE{\lvert h(X)-y_s(X)\rvert \mid S=s} 
    = \max_{s\in \{0,1\}} \EE{\lvert h_s^*(X)-y_s(X)\rvert \mid S=s}
    \leq 2\lambda.
\end{align}
Combining \eqref{eq::h_ys_h10_lamb_TV} with \eqref{eq::max_inf_lamb} gives 
\begin{align*}
    \epsilon_{\text{split}}^{\mathcal{H}}
    \geq \frac{1}{2} \max_{s\in \{0,1\}} \EE{\lvert h_1^*(X) - h_0^*(X)\rvert \mid S=s}
    - \TV(P_{0}\| P_{1}) - 3\lambda.
\end{align*}
\end{proof}

\subsection{Proof of Proposition~\ref{prop::NN_Barron_ben_split_ub}}
\label{append::NN_Barron}

\begin{proof}
By the triangle inequality, $\inf_{h\in \mathcal{H}}\max_{s\in \{0,1\}} \EE{\lvert h(X) - y_s(X)\rvert \mid S=s} \leq {\rm I}+{\rm II}$ where
\begin{align*}
    &{\rm I} \defined \inf_{h\in \mathcal{H}}\max_{s\in \{0,1\}} \EE{\lvert h(X) - h^*(X)\rvert \mid S=s}, \\
    &{\rm II} \defined \max_{s\in \{0,1\}}\EE{\lvert h^*(X) - y_s(X)\rvert \mid S=s},
\end{align*}
and $h^*$ is defined in \eqref{eq::h*_dP0dP}. Since $\max\{a,b\}\leq a+b$, we have ${\rm I} \leq 2 \inf_{h\in \mathcal{H}} \EE{|h(\bar{X}) - h^*(\bar{X})|}$
where the random variable $\bar{X}$ follows the probability distribution $(P_0 + P_1)/2$. By Barron's approximation bounds \citep{barron1993universal}, 
\begin{align}
    \inf_{h\in \mathcal{H}} \EE{|h(\bar{X}) - h^*(\bar{X})|} 
    \leq \frac{\mathsf{diam}(\mathcal{X})C}{\sqrt{k}},
\end{align}
where $C\defined \int_{\Reals^d} \|w\|_2 |\tilde{h^*}(w)|\dif w$, $\tilde{h^*}(w) \defined \frac{1}{(2\pi)^d} \int_{\mathcal{X}} h^*(x) \exp(-iwx)\dif x$. Moreover, by the proof of Theorem~\ref{thm::ben_split_l12_loss_bounds} (see Appendix~\ref{sec::proof_bounds_eps_12}), we have 
\begin{align*}
    {\rm II} \leq \min_{s\in\{0,1\}} \sqrt{\EE{(y_1(X)-y_0(X))^2\mid S=s}} \cdot \sqrt{1-\TV(P_0\| P_1)}.
\end{align*}
To summarize, if the hypothesis class contains feedforward neural network models with one layer of sigmoidal functions, the $\mathcal{H}$-benefit-of-splitting has an upper bound below. 
\begin{align*}
    \epsilon_{\text{split}}^{\mathcal{H}}
    &= \inf_{h\in \mathcal{H}}\max_{s\in \{0,1\}} \EE{\lvert h(X)-y_s(X)\rvert \mid S=s} - \max_{s\in \{0,1\}}\inf_{h\in \mathcal{H}} \EE{\lvert h(X)-y_s(X)\rvert \mid S=s}\\
    &\leq \min_{s\in\{0,1\}} \sqrt{\EE{(y_1(X)-y_0(X))^2\mid S=s}} \cdot \sqrt{1-\TV(P_0\| P_1)} + \frac{2\mathsf{diam}(\mathcal{X})C}{\sqrt{k}}.
\end{align*}
\end{proof}

\subsection{Proof of Corollary~\ref{cor::simp_mod_tradeoff}}
\label{append::H_ben_split_geq_1/2}

We approach Corollary~\ref{cor::simp_mod_tradeoff} by proving a more general result.
\begin{lem}
For any hypothesis $\mathcal{H}$, there exists a probability distribution $Q_{S,X,Y}$ whose $\mathcal{H}$-benefit-of-splitting is at least
\begin{align*}
    \frac{1}{2} \sup_{\substack{h_1,h_0\in \mathcal{H}\\ x\in\mathcal{X}}} |h_1(x) - h_0(x)|.
\end{align*}
\end{lem}
\begin{proof}
For any $\epsilon > 0$, there exist two classifiers $h_1^*, h_0^*\in\mathcal{H}$ and $x^*\in \mathcal{X}$ such that 
\begin{align}
\label{eq::cons_class_proof_cor_simp_mod}
    |h_1^*(x^*) - h_0^*(x^*)|
    \geq \sup_{\substack{h_1,h_0\in \mathcal{H}\\ x\in\mathcal{X}}} |h_1(x) - h_0(x)| - \epsilon.
\end{align}
Now we construct a probability distribution $Q_{S,X,Y}$ with $Q_{Y|X,S}(1|x,s) = h_s^*(x),$ $Q_{X|S=s}(x) = \delta(x-x^*)$, $Q_S(s) = 0.5$ for $s\in \{0,1\}$ where $\delta(\cdot)$ is the Dirac delta function. Our lower bound in  Theorem~\ref{thm::H_dep_split_bounds} implies that $\epsilon_{\text{split}}^{\mathcal{H}} \geq  \frac{1}{2}|h_1^*(x^*) - h_0^*(x^*)|$ 
which, due to \eqref{eq::cons_class_proof_cor_simp_mod}, can be further lower bounded by $\frac{1}{2} (\sup_{h_1,h_0\in \mathcal{H}, x\in\mathcal{X}} |h_1(x) - h_0(x)| - \epsilon )$. Since this lower bound of $\epsilon_{\text{split}}^{\mathcal{H}}$ holds for any $\epsilon>0$, one can let $\epsilon$ be sufficiently small which leads to the desired conclusion. 
\end{proof}

\subsection{Proof of Theorem~\ref{thm::samp_comp}}
\label{append::samp_comp}

We introduce the empirical benefit-of-splitting and bound its difference from the sample-limited-splitting (see Definition~\ref{defn::emp_ben_split}).
\begin{defn}
\label{defn::train_ben_split}
For a given hypothesis class $\mathcal{H}$ and $n_s$ i.i.d. samples $\{(x_{s,i},y_{s,i})\}_{i=1}^{n_s}$ from group $s\in \{0,1\}$, the empirical-splitting is defined as 
\begin{align}
    \hat{\epsilon}_{\text{split}, \text{emp}}
    \defined \inf_{h\in\mathcal{H}}\max_{s\in \{0,1\}} \frac{\sum_{i=1}^{n_s} |h(x_{s,i}) - y_{s,i}|}{n_s} - \max_{s\in \{0,1\}} \inf_{h\in\mathcal{H}} \frac{\sum_{i=1}^{n_s} |h(x_{s,i}) - y_{s,i}|}{n_s}.
\end{align}
\end{defn}
\begin{lem}
\label{lem::split_emp_train_disc}
Let $\mathcal{H}$ be a hypothesis class from $\mathcal{X}$ to $\{0,1\}$ with VC dimension $D$. Then with probability at least $1-\delta$,
\begin{align}
    \left|\hat{\epsilon}_{\text{split}} - \hat{\epsilon}_{\text{split}, \text{emp}}  \right|
    \leq 4\max_{s\in \{0,1\}} \sqrt{\frac{2D\log(6n_s) + 2\log(16/\delta)}{n_s}},
\end{align}
where $n_s$ is the number of samples from group $s\in \{0,1\}$.
\end{lem}

\begin{proof}
Corollary 3.8 and Theorem~4.3 in \citep{anthony2009neural} together imply that with probability at least $1-\delta$, for any $s\in \{0,1\}$ and $h\in \mathcal{H}$, 
\begin{align*}
    \left|\frac{\sum_{i=1}^{n_s} |h(x_{s,i}) - y_{s,i}|}{n_s} - \EE{\lvert h(X)-y_s(X)\rvert \mid S=s}\right|
    \leq 2\sqrt{\frac{2D\log(6n_s) + 2\log(8/\delta)}{n_s}}.
\end{align*}
Therefore, for any $h_s \in \mathcal{H}$ with $s\in \{0,1\}$ 
\begin{align}
\label{eq::samp_comp_dif}
    \left|\max_{s\in \{0,1\}} \frac{\sum_{i=1}^{n_s} |h_s(x_{s,i}) - y_{s,i}|}{n_s} - \max_{s\in \{0,1\}} \EE{\lvert h_s(X)-y_s(X)\rvert \mid S=s}\right| 
    \leq 2\max_{s\in \{0,1\}} \sqrt{\frac{2D\log(6n_s) + 2\log(8/\delta)}{n_s}}.
\end{align}
Recall that 
\begin{align*}
    \hat{\epsilon}_{\text{split}} 
    = \max_{s\in \{0,1\}} \EE{\lvert \hat{h}^*(X)-y_s(X)\rvert \mid S=s} - \max_{s\in \{0,1\}} \EE{\lvert \hat{h}^*_s(X)-y_s(X)\rvert \mid S=s}.
\end{align*}
Now by \eqref{eq::samp_comp_dif}, we conclude that 
\begin{align*}
    \left|\hat{\epsilon}_{\text{split}} - \hat{\epsilon}_{\text{split}, \text{emp}}  \right|
    \leq 4\max_{s\in \{0,1\}} \sqrt{\frac{2D\log(6n_s) + 2\log(8/\delta)}{n_s}}.
\end{align*}
\end{proof}
Since the upper and lower bounds of $\epsilon_{\text{split}}^{\mathcal{H}}$ (see Theorem~\ref{thm::H_dep_split_bounds}) hold for any underlying distribution $P_{S,X,Y}$. One can plug in the empirical distribution and obtain the corresponding bounds for $\hat{\epsilon}_{\text{split}, \text{emp}}$. Then we obtain the desired bounds for $\hat{\epsilon}_{\text{split}}$ by using Lemma~\ref{lem::split_emp_train_disc} for bounding the difference between $\hat{\epsilon}_{\text{split}, \text{emp}}$ and $\hat{\epsilon}_{\text{split}}$.

\section{Supporting Results for Experiments}
%\section{Proofs for Section~\ref{sec::exp}}

\subsection{Closed-form Expression of $\epsilon_{\text{split},\mathsf{FER}}$}
\label{append::syn_FER_ben_split}

\begin{proof}
For the distributions we construct, one can choose $\{y_s(x)\}_{s\in \{0,1\}}$ as the split classifiers which lead to zero false error rate. Therefore, the problem remains computing the false error rate of the optimal group-blind classifier. First, the labeling functions naturally divide $\Reals^2$ into four parts: ${\rm I} \defined \{x \mid y_0(x)=1, y_1(x)=1\}$, ${\rm II} \defined \{x \mid y_0(x)=1, y_1(x)=0\}$, ${\rm III}\defined \{x \mid y_0(x)=0, y_1(x)=0\}$, ${\rm IV}\defined \{x \mid y_0(x)=0, y_1(x)=1\}$. Clearly, in order to be an optimal group-blind classifier, $h$ must satisfy $h(x) = 1$ on ${\rm I}$ and $h(x) = 0$ on ${\rm III}$. We define $h|_{{\rm II}}$ and $h|_{{\rm IV}}$ as
\begin{align*}
    h|_{{\rm II}}(x) 
    \defined 
    \begin{cases}
        h(x) &\text{ if } x\in {\rm II}\\
        0 &\text{ otherwise},
    \end{cases}
    \quad 
    h|_{{\rm IV}}(x) 
    \defined
    \begin{cases}
        h(x) &\text{ if } x\in {\rm IV}\\
        0 &\text{ otherwise}.
    \end{cases}
\end{align*}
Due to our construction of the labeling functions and Lemma~\ref{lem::FER_equivalent_exp},
\begin{align}
    \epsilon_{\text{split},\mathsf{FER}}
    &= \inf_{h:\mathcal{X}\to [0,1]}\max\left\{ \frac{\EE{h(X)(1-y_0(X))\mid S=0}}{\Pr(Y=0 \mid S=0)},\ 1- \frac{\EE{h(X)y_0(X)\mid S=0}}{\Pr(Y=1\mid S=0)}, \right. \nonumber\\  
    &\quad \quad \quad \quad \quad \quad \quad \ \  \left.\frac{\EE{h(X)(1-y_1(X))\mid S=1}}{\Pr(Y=0 \mid S=1)},\ 1- \frac{\EE{h(X)y_1(X)\mid S=1}}{\Pr(Y=1\mid S=1)} \right\} \nonumber\\
    &= \inf_{h:\mathcal{X}\to [0,1]}\max\left\{ 2\EE{h|_{{\rm IV}}(X)\mid S=0},\ 1 - 2\Pr(X\in {\rm I}\mid S=0) - 2\EE{h|_{{\rm II}}(X)\mid S=0}, \right. \nonumber\\  
    &\quad \quad \quad \quad \quad \quad \quad \  \left. 2\EE{h|_{{\rm II}}(X)\mid S=1},\  1 - 2\Pr(X\in {\rm I}\mid S=1) - 2\EE{h|_{\rm IV}(X)\mid S=1} \right\} \nonumber\\
    &= \max\left\{\inf_{h|_{{\rm II}}:\mathcal{X}\to [0,1]}\max\left\{1 - 2\Pr(X\in {\rm I}\mid S=0) - 2\EE{h|_{{\rm II}}(X)\mid S=0},\ 2\EE{h|_{{\rm II}}(X)\mid S=1}\right\},\right. \nonumber\\
    &\quad \quad \quad \ \ \left. \inf_{h|_{{\rm IV}}:\mathcal{X}\to [0,1]}\max\left\{ 2\EE{h|_{{\rm IV}}(X)\mid S=0},\ 1 - 2\Pr(X\in {\rm I}\mid S=1) - 2\EE{h|_{\rm IV}(X)\mid S=1} \right\}\right\} \nonumber\\
    &= \inf_{h|_{{\rm IV}}:\mathcal{X}\to [0,1]}\max\left\{ 2\EE{h|_{{\rm IV}}(X)\mid S=0},\ 1 - 2\Pr(X\in {\rm I}\mid S=1) - 2\EE{h|_{\rm IV}(X)\mid S=1} \right\}, \label{eq::eps_FER_split_Gaussian_infmax}
\end{align}
where the last step is because of symmetry. Since $2\max\{a,b\}\geq a+b$, $\epsilon_{\text{split},\mathsf{FER}}$ can be lower bounded by 
\begin{align}
    &\frac{1}{2} - \Pr(X\in {\rm I}\mid S=1) + \inf_{h|_{{\rm IV}}:\mathcal{X}\to [0,1]} \EE{h|_{{\rm IV}}(X)\mid S=0} - \EE{h|_{\rm IV}(X)\mid S=1} \nonumber\\
    =& \frac{1}{2} - \Pr(X\in {\rm I}\mid S=1) + \inf_{h|_{{\rm IV}}:\mathcal{X}\to [0,1]}  \int_{{\rm IV}} h|_{{\rm IV}}(x) (\dif P_0(x) - \dif P_1(x)) \nonumber\\
    =& \frac{1}{2} - \Pr(X\in {\rm I}\mid S=1) - \int_{{\rm IV}} (\dif P_1(x) - \dif P_0(x))_{+}. \nonumber 
\end{align}
Since $P_0 \sim \mathcal{N}\left(\mathbf{0}, \bSigma_0\right)$ and $P_0 \sim \mathcal{N}\left(\mathbf{0}, \bSigma_1\right)$, by comparing their probability density functions, we have 
\begin{align*}
    \mathcal{A}
    \defined \{x \in {\rm IV} \mid \dif P_1(x) > \dif P_0(x) \}
    = \{x \in {\rm IV} \mid x_2 < 0\}
    = \{x\in\Reals^2\mid y_1(x) = 1, x_2<0\}.
\end{align*}
Therefore, $\epsilon_{\text{split},\mathsf{FER}}$ can be lower bounded by 
\begin{align}
    \frac{1}{2}\left(1 - 2\Pr(X\in {\rm I}\mid S=1) - 2\Pr(X \in \mathcal{A} \mid S=1) + 2\Pr(X \in \mathcal{A} \mid S=0) \right). \label{eq::lb_eps_split_FER_average}
\end{align}
By symmetry, we have 
\begin{align*}
    \Pr(X\in {\rm I}\mid S=1)
    = \Pr(X\in {\rm III}\mid S=1),\quad
    \Pr(X\in {\rm II}\mid S=1)
    = \Pr(X\in {\rm IV}\mid S=1),
\end{align*}
which leads to
\begin{align}
    1 - 2\Pr(X\in {\rm I}\mid S=1) - 2 \Pr(X\in \mathcal{A}\mid S=1)
    = 2\Pr(X\in {\rm IV}\backslash \mathcal{A}\mid S=1)
    = 2\Pr(X\in \mathcal{A}\mid S=0), \label{eq::Gaussian_sym_prop}
\end{align}
where the last step is by symmetry again. Therefore, the lower bound of $\epsilon_{\text{split},\mathsf{FER}}$ in \eqref{eq::lb_eps_split_FER_average} can be simplified as $2\Pr(X\in \mathcal{A}\mid S=0)$. On the other hand, we can design a classifier $h|_{{\rm IV}}^*(x) = 1$ if $x\in \mathcal{A}$; $h|_{{\rm IV}}^*(x) = 0$ otherwise. By \eqref{eq::eps_FER_split_Gaussian_infmax}, we have 
\begin{align}
    \epsilon_{\text{split},\mathsf{FER}}
    &\leq \max\left\{ 2\EE{h|_{{\rm IV}}^*(X)\mid S=0},\ 1 - 2\Pr(X\in {\rm I}\mid S=1) - 2\EE{h|_{\rm IV}^*(X)\mid S=1} \right\} \nonumber\\
    &= \max\left\{2 \Pr(X\in \mathcal{A}\mid S=0),\ 1 - 2\Pr(X\in {\rm I}\mid S=1) - 2 \Pr(X\in \mathcal{A}\mid S=1) \right\}. \nonumber
\end{align}
Due to \eqref{eq::Gaussian_sym_prop}, the above upper bound of $\epsilon_{\text{split},\mathsf{FER}}$ is equal to
$2\Pr(X\in \mathcal{A}\mid S=0)$ as well. Hence, $\epsilon_{\text{split},\mathsf{FER}}=2\Pr(X\in \mathcal{A}\mid S=0)$.
\end{proof}

\subsection{Total Variation Distance Estimation}
\label{append::real_data_TV}

We provide details on how we estimate the total variation distance $\TV(P_{0}\| P_{1})$ by using $n_s$ i.i.d. unlabeled data $\{x_{s,i}\}_{i=1}^{n_s}$ drawn from each group $s\in \{0,1\}$. By applying Baye's rule, we can write the density ratio equivalently as 
\begin{align*}
    \frac{\dif P_1(x)}{\dif P_0(x)}
    = \frac{\dif P_{X|S=1}(x)}{\dif P_{X|S=0}(x)}
    = \frac{\Pr(S=1\mid X=x)}{1 - \Pr(S=1\mid X=x)}\cdot \frac{1-\Pr(S=1)}{\Pr(S=1)},
\end{align*}
which leads to:
\begin{align}\label{eq:TV_expression}
    \TV(P_0\| P_1) = \int \indicator{\left[\tfrac{\Pr(S=1\mid X=x)}{1 - \Pr(S=1\mid X=x)} \geq \tfrac{\Pr(S=1)}{1 - \Pr(S=1)}\right]} \dif P_1(x) -  \int \indicator{\left[\tfrac{\Pr(S=1\mid X=x)}{1 - \Pr(S=1\mid X=x)} \geq \tfrac{\Pr(S=1)}{1-\Pr(S=1)}\right]} \dif P_0(x).
\end{align}
This expression gives rise to the following procedure of estimating the total variation distance. 
\begin{itemize}
    \item Compute a constant $\alpha=\frac{n_1}{n_0 + n_1}$ to estimate the marginal probability $\Pr(S=1)$ and train a classifier $s(x)$ to approximate the conditional distribution $\Pr(S=1\mid X=x)$. In particular, we use a feed-forward neural network for $s(x)$, which consists of one hidden layer with $100$ neurons and \texttt{ReLU} activation, and a soft-max readout layer.
    We adopt cross entropy as the loss function, set learning rate to be $0.001$, and use \texttt{AdamOptimizer} \citep{kingma2014adam} to train the datasets with batch size $200$.
    To avoid overfitting, we hold $10\%$ of the samples as a validation set, and terminate training once the validation loss is not improving by $10^{-4}$ for the next $10$ consecutive epochs (i.e., early stopping), and the maximum number of epochs is set to be $200$.
    \item By plugging $\alpha$ and $s(x)$ into \eqref{eq:TV_expression} and using i.i.d. samples to estimate the integrals (i.e., expectations), we obtain the following approximation of $\TV(P_0 \| P_1)$: 
    \begin{align}
        \frac{1}{n_1}\sum_{i=1}^{n_1} \indicator{\left[\frac{s(x_{1,i})}{1-s(x_{1,i})} \geq \frac{\alpha}{1-\alpha} \right]}
        - \frac{1}{n_0}\sum_{i=1}^{n_0} \indicator{\left[\frac{s(x_{0,i})}{1-s(x_{0,i})}\geq \frac{\alpha}{1-\alpha}\right]}.
    \end{align}
\end{itemize}
We remark that estimating information-theoretic measures has been studied in \citep[e.g.,][]{kanamori2011f,belghazi2018mine,hsu2019obfuscation,jiao2018minimax,wu2016minimax}.

\end{document}